\documentclass{article}

\PassOptionsToPackage{sort,numbers,compress}{natbib}
\usepackage[preprint]{neurips_2026}

\usepackage[utf8]{inputenc}
\usepackage[T1]{fontenc}
\usepackage{hyperref}
\usepackage{url}
\usepackage{booktabs}
\usepackage{amsmath,amssymb,amsfonts}
\usepackage{amsthm}
\usepackage{nicefrac}
\usepackage{microtype}
\usepackage[table]{xcolor}
\usepackage{graphicx}
\usepackage{subcaption}
\usepackage{enumitem}
\usepackage{multirow}
\usepackage{makecell}
\usepackage[most]{tcolorbox}
\usepackage{float}
\usepackage{wrapfig}
\usepackage{placeins}
\usepackage{pgfplots}
\usepgfplotslibrary{groupplots,fillbetween}
\usetikzlibrary{positioning}
\pgfplotsset{compat=1.18}

\newcommand{\ours}{\textsf{HYDRA}}

\newcommand{\etal}{\textit{et al.}}

\newcommand{\miou}{\mathrm{mIoU}}

 \theoremstyle{plain}
 \newtheorem{theorem}{Theorem}[section]
 \newtheorem{proposition}[theorem]{Proposition}
 \theoremstyle{definition}
 
 \theoremstyle{remark}
 
 \definecolor{keyblue}{RGB}{30,80,160}
 \definecolor{keyyellow}{RGB}{255,243,205}
 \definecolor{keygreen}{RGB}{220,240,220}
 \definecolor{keyred}{RGB}{240,220,220}
 \definecolor{headerblue}{RGB}{224,235,248}
 \definecolor{tableblue}{RGB}{235,244,252}
 \definecolor{sectionblue}{RGB}{247,250,253}

 \newtcolorbox{keyquestion}{
   colback=keyyellow, colframe=keyblue,
   fonttitle=\bfseries\small, title=Key Question,
   left=4pt, right=4pt, top=3pt, bottom=3pt,
   boxrule=1pt, arc=3pt
 }

 \newtcolorbox{keyfinding}{
   colback=keygreen, colframe=black!40,
   fonttitle=\bfseries\small, title=Key Finding,
   left=4pt, right=4pt, top=3pt, bottom=3pt,
   boxrule=0.5pt, arc=3pt
 }

 \newtcolorbox{takeaway}{
   colback=keyblue!8, colframe=keyblue,
   fonttitle=\bfseries\small, title=Takeaway,
   left=4pt, right=4pt, top=3pt, bottom=3pt,
   boxrule=1pt, arc=3pt
 }

 \newtcolorbox{motivation}{
   colback=keyred!60, colframe=black!40,
   fonttitle=\bfseries\small, title=Motivation Takeaway,
   left=4pt, right=4pt, top=3pt, bottom=3pt,
   boxrule=0.5pt, arc=3pt
 }

  \usepackage{tcolorbox}
 \newtcolorbox{insight}{
    colback=keyblue!4,
    colframe=keyblue,
    boxrule=0pt,
    leftrule=1pt,
    arc=0pt,
    top=1.5pt, bottom=1.5pt,
    left=4pt, right=3pt,
    boxsep=1.5pt,
    before skip=3pt, after skip=3pt,
    fontupper=\small
	  }

  \newtcolorbox{appendixroadmap}{
    colback=sectionblue,
    colframe=keyblue!70!black,
    boxrule=0.6pt,
    arc=1pt,
    top=4pt, bottom=4pt,
    left=5pt, right=5pt,
    before skip=4pt, after skip=6pt
  }

\title{Queries Knew More Than We Thought: Uncovering Latent Knowledge in Segmentation Models}

\author{
Ignacio M. De la Jara$^{1,3}$ \and
Cristian Rodriguez-Opazo$^{2}$ \and
Damith Ranasinghe$^{1,3}$\\[1ex]
$^{1}$School of Computer and Mathematical Sciences, The University of Adelaide, Australia\\
$^{2}$School of Computing, Australian National University, Australia\\
$^{3}$Australian Institute for Machine Learning (AIML), Australia\\[1ex]
\texttt{ignacio.mezadelajara@adelaide.edu.au}
}

\begin{document}

\maketitle

\begin{abstract}
Modern segmenters often fail after the expensive computation has already
been done: a useful mask is present among the model's query-conditioned
candidates, but the deployed selection rule does not expose it. We study
this output-selection bottleneck in frozen DETR-family. A ground-truth-only
oracle first shows substantial hidden headroom in already-computed mask
proposals. This raises a simple question:
textit{How can we better use the masks a segmenter has already computed but does not expose?}
We then ask whether that headroom can be recovered without
adding queries, generating new masks, rerunning the backbone, or updating
weights. \ours{} is a small selector trained only on cached frozen outputs. At inference time, it scores the cached candidates against an
explicit keep-baseline option and acts only when a held-out calibrated
margin says the selected candidate is sufficiently better. Trained on
training-split caches and calibrated on held-out data, \ours{} improves
Mask2Former, MaskDINO, and OneFormer by up to $+7.41$ dataset mIoU points
on ADE20k and COCO, and improves SAM~3 by $+9.4$ class-macro prompt-IoU
points on average across eight domains, while preserving useful predictions through calibration. Paired LoRA controls show that light weight adaptation does not remove the bottleneck: exposed predictions are often flat or
worse, while routing over the adapted candidates still recovers accuracy.
Finally, we connect the effect to query specialization under bipartite
matching and verify it in a controlled TinyDETR study. These results show
that frozen segmenters should be evaluated not only by the masks they
expose, but also by the useful candidates they suppress.
\end{abstract}


\begin{figure}[t]
    \centering
    \includegraphics[
        width=\linewidth,
        trim={20pt 0 20pt 0},
        clip
    ]{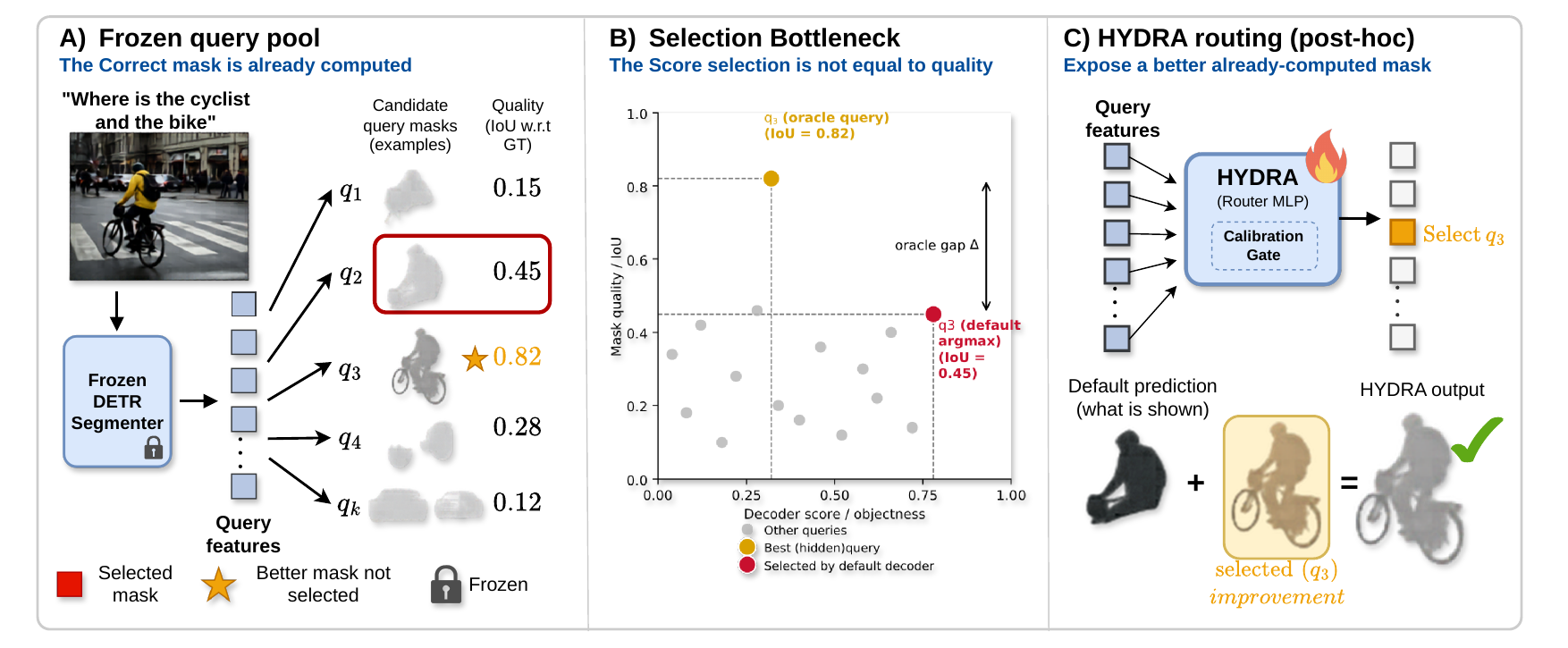}
    \vspace{-0.6em}
    \caption{\textbf{Models know more than they show.}
    A frozen segmenter already computes candidate masks that the default
    selection rule can suppress. \ours{}
    learns a calibrated post-hoc route over frozen query features and
    exposes a better, already-computed, mask without changing the backbone. 
    }
    \label{fig:main_overview}
\end{figure}

\section{Introduction}
\label{sec:intro}

Modern universal segmenters \cite{cheng2022mask2former} and promptable foundation models have achieved strong performance by learning rich visual representations at scale. Yet, their deployment often falls short of their true potential. We discovered that these failures are frequently not due to limitations in representation or mask generation, but instead arise from \emph{selection} failures. In many cases, correct masks are already present among a model’s query outputs but are suppressed by the selection mechanism used at inference time.

Notably, DETR-family segmenters first emit query-conditioned candidates, then collapse the set through a fixed final selection rule~\citep{carion2020detr,cheng2022mask2former,li2023maskdino,jain2023oneformer,carion2025sam3}. 
After the expensive forward pass to generate candidate masks, the deployed selection rule can suppress the candidate that best matches the target. This observation has important implications. 

\begin{insight}
It suggests that improving segmentation does not necessarily require retraining models, but can instead be achieved by better utilization of their existing outputs.
\end{insight}
Consequently, we treat the candidate mask set as the object of study and investigate:
\begin{center}
    \textit{What potential performance gains exist? What do current selectors hide? and Can the latent knowledge be recovered safely to improve performance?}
\end{center}
We systematically investigate this phenomenon with an oracle selector and demonstrate that substantial, consistent performance headroom exists across DETR-based universal segmenters (e.g., Mask2Former, MaskDINO, OneFormer) and promptable models such as SAM3. Our oracle evaluation reveal the frozen candidate set can always deliver a better prediction than the default selection methods. This suggests that the suppressed performance headroom can be recovered without adding queries, generating new masks, rerunning the backbone, or updating weights~\citep{guo2017calibration,huang2019maskscoring,cai2023aligndetr,pu2023rankdetr,munir2024caldetr}. 

The structural reason is simple. Bipartite matching gives each query gradients only from its assigned objects, encouraging specialization. Once queries specialize, the best query becomes input-dependent, and a static threshold or score argmax cannot be a universal quality ranking.
Section~\ref{sec:theory} formalizes this claim and Appendix~\ref{app:controlled} verifies it in TinyDETR.

Motivated by this insight, we rethink selection and propose learning a simple, deployable selector that operates on frozen model outputs, we dub \ours. The routing action in \ours{} evaluates candidate proposals form our learned selector against the default prediction---included an explicit \emph{keep-baseline} option---and uses calibrated abstention to avoid harmful changes. This design ensures that improvements are realized when possible while preserving strong baseline behavior and safety.

\vspace{0.5em}
\noindent\textbf{Contributions.~}In summary:
\begin{enumerate}[leftmargin=*,topsep=0pt,itemsep=0pt,parsep=0pt]
    \item A frozen-output oracle diagnostic showing hidden selection headroom across universal segmenters and SAM~3.
    \item \ours{}, a calibrated, learned selector that improves every model average using only frozen model outputs. 
    \item A \textbf{theoretical and empirical analysis} linking the gap to query specialization induced by bipartite matching. And through controlled experiments, showing that the necessary routing signal already exists in frozen query features. 
    \item Safety, adaptation, and mechanism evidence: lower-tail and
    harmful-action diagnostics, paired LoRA controls, and a matching-based
    account of input-dependent query choice.
\end{enumerate}

Overall, our findings provide both practical gains and conceptual insights. 
\section{Why Query Routing Should Exist}
\label{sec:theory}
 
Query-based segmenters produce $K$ candidate masks and collapse them
with a fixed rule. Prior work mostly improves how queries are generated;
we ask whether the final rule uses the generated queries well. The
answer should not be assumed. Bipartite matching creates specialized,
non-exchangeable queries. Once specialization appears, the best query
depends on the input, and decoding becomes routing rather than scalar
thresholding.

\subsection{Formalization}
\label{sec:formalization}
 
Let a query-based model $f_\theta$ with $K$ queries produce, for input
$x$:
\begin{equation}
    f_\theta(x)
    =
    \bigl(\,
        z(x),\;\;
        \{\hat{m}_k(x)\}_{k=1}^K,\;\;
        \{\hat{s}_k(x)\}_{k=1}^K
    \,\bigr),
    \label{eq:model}
\end{equation}
where $z(x)\!\in\!\mathbb{R}^d$ is the encoder representation,
$\hat{m}_k(x)\!\in\![0,1]^{H\times W}$ the $k$-th mask, and
$\hat{s}_k(x)\!\in\![0,1]$ its score.
We measure per-query quality as
$u_k(x)=\mathrm{IoU}(\hat{m}_k(x),m^*(x))$.
A decoding rule $\phi\!:\!\mathcal{X}\!\to\![K]$ achieves
$\mathcal{R}(\phi)=\mathbb{E}_x[u_{\phi(x)}(x)]$.
The oracle $k^*(x)=\arg\max_k u_k(x)$ attains
$\mathcal{R}^*=\mathbb{E}_x[\max_k u_k(x)]$, and the
\emph{oracle gap} is
$\Delta(\phi)=\mathcal{R}^*-\mathcal{R}(\phi)\geq 0$.
 
\begin{proposition}[Gap decomposition]
\label{prop:decomp}
For any decoding rule $\phi$:
\begin{equation}
    \Delta(\phi)
    =
    \mathbb{E}_x\!\Bigl[
        \mathbf{1}[\phi(x)\neq k^*(x)]
        \;\cdot\;
        \bigl(U^*(x)-u_{\phi(x)}(x)\bigr)
    \Bigr],
    \quad U^*(x)=\max_k u_k(x).
    \label{eq:decomp}
\end{equation}
\end{proposition}
\begin{proof}
Partition $\mathbb{E}_x[U^*\!-u_\phi]$ on $\{\phi=k^*\}$;
the term vanishes when $\phi(x)=k^*(x)$.
\end{proof}
 
\subsection{Structural Inevitability}
\label{sec:inevitability}
 
Equation~\eqref{eq:decomp} shows when a gap appears; bipartite matching
makes such gaps natural.
 
\paragraph{Query heterogeneity.}
Given a partition $\{\mathcal{X}_j\}_{j=1}^J$ with
$\mu(\mathcal{X}_j)>0$, define
$S_{kj}=\mathbb{E}_{x|\mathcal{X}_j}[u_k(x)]$.
The model has \emph{decoding-relevant query heterogeneity} if no single
query is conditionally optimal on every cell:
for every $\bar{k}\in[K]$, there exists a cell $j$ such that
$S_{\bar{k}j}<\max_k S_{kj}$. Different regions of the input space then
require different queries.
 
\begin{proposition}[Fixed rules are suboptimal]
\label{prop:fixed}
Decoding-relevant heterogeneity implies $\Delta(\phi)>0$ for every
fixed rule $\phi\equiv\bar{k}$.
\end{proposition}
\begin{proof}
For any cell $j$,
\begin{equation}
\mathbb{E}[U^*(x)-u_{\bar{k}}(x)\mid x\in\mathcal{X}_j]
\geq
\max_k S_{kj}-S_{\bar{k}j}.
\end{equation}
By decoding-relevant heterogeneity, the right-hand side is strictly
positive for at least one cell for every $\bar{k}$. Since that cell has
positive measure, $\Delta(\phi)=\mathbb{E}[U^*-u_{\bar{k}}]>0$.
\end{proof}
 
\begin{proposition}[Score argmax is suboptimal]
\label{prop:score}
The score head is trained with binary detection loss
$\mathcal{L}_{\mathrm{score}}
=-\sum_k[c^*_k\log\hat{s}_k+(1\!-\!c^*_k)\log(1\!-\!\hat{s}_k)]$,
which rewards object \emph{presence}, not relative query
\emph{quality}. Whenever $P(\phi_s\neq k^*)>0$ and
$\mathbb{E}[U^*\!-\!u_{\phi_s}\mid\phi_s\neq k^*]>0$, then
$\Delta(\phi_s)>0$.
\end{proposition}
 
\paragraph{Why training induces heterogeneity.}
Writing $\pi_t(n)=k$ when object $n$ is assigned to query $k$ at
training step $t$, the gradient for query $k$ is
$\nabla_{q_k}\mathcal{L}
=\sum_t\sum_n\mathbf{1}[\pi_t(n)=k]\cdot
\nabla_{q_k}\mathcal{L}_{\mathrm{mask}}(q_k,m^*_{n})$.
Each query receives gradients \emph{only} from matched objects. Chance
specialization can therefore change future assignments, deepening the
specialization. When this feedback yields decoding-relevant
heterogeneity, Proposition~\ref{prop:fixed} applies. The argument depends
on selective gradient flow, not on a particular architecture.

\begin{figure}[t]
    \centering
    \includegraphics[width=0.98\linewidth]{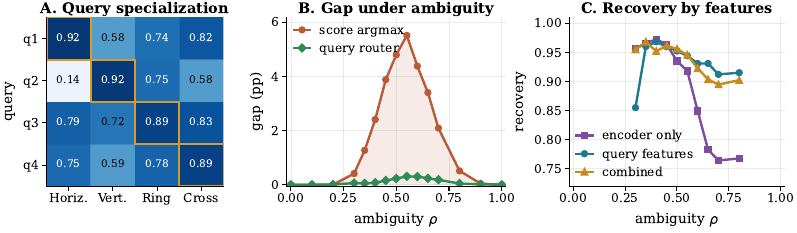}
    \vspace{-0.6em}
    \caption{\textbf{A compact controlled example.}
    In the $K{=}4$ TinyDETR stress case, score argmax opens an oracle gap
    near the ambiguous boundary; a query-feature router stays close to the
    oracle, while an encoder-only router loses recovery.}
    \label{fig:toy_routing_mechanism}
    \vspace{-0.8em}
\end{figure}

\paragraph{Controlled toy example.}
We train TinyDETR on $64{\times}64$ synthetic shape masks, freeze it, and
sweep an ambiguity parameter between training regimes.
Figure~\ref{fig:toy_routing_mechanism} compares score argmax with an
oracle over the same candidate masks, a query-feature router, and an
encoder-only router. Hungarian matching makes different queries best for
different shape families. At the interpolation boundary, scores remain
confident about object presence but stop ranking masks by IoU. The
query-feature router recovers most lost quality; the encoder-only router
degrades where query choice is ambiguous. The main figure shows the
$K{=}4$ stress case; full protocol, matching $K{=}2$ and $K{=}3$ sweeps,
boundary-gap decomposition, and query-feature router ablations are in
Appendix~\ref{app:controlled}. Deferred proofs for the formal claims are
in Appendix~\ref{app:proofs}.
 
\begin{insight}
\textbf{The practical consequence is simple:} if the best query changes with the
input, then the final decoder must solve a routing problem. Scores can
still be useful, but they are not guaranteed to rank counterfactual mask
quality. A post-hoc selector is therefore not an extra segmentation
model; it is a way to read out information the frozen query pool already
contains.
\end{insight}


\section{Oracle Analysis: What Is Already Computed?}
\label{sec:oracle_analysis}
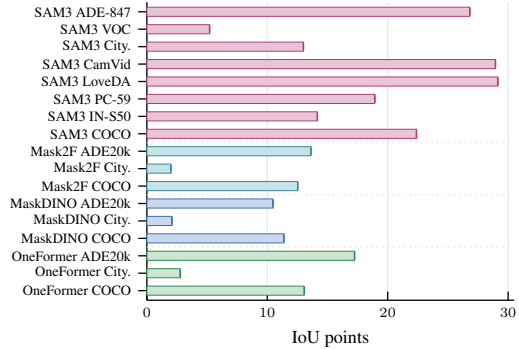
\begin{wrapfigure}[14]{r}{0.54\textwidth}
    \vspace{-0.8em}
    \centering
    \resizebox{\linewidth}{!}{\definecolor{hydraGrid}{HTML}{E3E3E3}
\definecolor{hydraSAMLight}{HTML}{E8C4D7}
\definecolor{hydraM2FLight}{HTML}{C7E3E7}
\definecolor{hydraMDINOLight}{HTML}{CAD8EE}
\definecolor{hydraOFLight}{HTML}{CBE5D0}
\definecolor{hydraSAM}{HTML}{B33C7A}
\definecolor{hydraM2F}{HTML}{2F8F9D}
\definecolor{hydraMDINO}{HTML}{3B6FB6}
\definecolor{hydraOF}{HTML}{2E8B57}

\begin{tikzpicture}
\begin{axis}[
  width=7.95cm,
  height=6.75cm,
  xbar,
  xmin=0,
  xmax=30.5,
  ymin=-0.55,
  ymax=16.55,
  xlabel={IoU points},
  ytick={16,15,14,13,12,11,10,9,8,7,6,5,4,3,2,1,0},
  yticklabels={SAM3 ADE-847,SAM3 VOC,SAM3 City.,SAM3 CamVid,SAM3 LoveDA,SAM3 PC-59,SAM3 IN-S50,SAM3 COCO,Mask2F ADE20k,Mask2F City.,Mask2F COCO,MaskDINO ADE20k,MaskDINO City.,MaskDINO COCO,OneFormer ADE20k,OneFormer City.,OneFormer COCO},
  xtick={0,10,20,30},
  axis x line*=bottom,
  axis y line*=left,
  axis line style={draw=black, line width=0.65pt},
  tick style={draw=black, line width=0.65pt},
  tick label style={font=\scriptsize, text=black},
  label style={font=\small, text=black},
  yticklabel style={font=\scriptsize, text=black},
  xmajorgrids=true,
  ymajorgrids=false,
  grid style={draw=hydraGrid, line width=0.4pt},
  bar width=4.5pt,
  clip=false,
]
\addplot+[xbar, bar shift=0pt, mark=none, fill=hydraSAMLight, draw=hydraSAM, line width=0.35pt] coordinates {(26.82,16)};
\draw[hydraSAM, line width=0.8pt] (axis cs:26.82,15.78) -- (axis cs:26.82,16.22);
\addplot+[xbar, bar shift=0pt, mark=none, fill=hydraSAMLight, draw=hydraSAM, line width=0.35pt] coordinates {(5.21,15)};
\draw[hydraSAM, line width=0.8pt] (axis cs:5.21,14.78) -- (axis cs:5.21,15.22);
\addplot+[xbar, bar shift=0pt, mark=none, fill=hydraSAMLight, draw=hydraSAM, line width=0.35pt] coordinates {(13.00,14)};
\draw[hydraSAM, line width=0.8pt] (axis cs:13.00,13.78) -- (axis cs:13.00,14.22);
\addplot+[xbar, bar shift=0pt, mark=none, fill=hydraSAMLight, draw=hydraSAM, line width=0.35pt] coordinates {(28.95,13)};
\draw[hydraSAM, line width=0.8pt] (axis cs:28.95,12.78) -- (axis cs:28.95,13.22);
\addplot+[xbar, bar shift=0pt, mark=none, fill=hydraSAMLight, draw=hydraSAM, line width=0.35pt] coordinates {(29.15,12)};
\draw[hydraSAM, line width=0.8pt] (axis cs:29.15,11.78) -- (axis cs:29.15,12.22);
\addplot+[xbar, bar shift=0pt, mark=none, fill=hydraSAMLight, draw=hydraSAM, line width=0.35pt] coordinates {(18.93,11)};
\draw[hydraSAM, line width=0.8pt] (axis cs:18.93,10.78) -- (axis cs:18.93,11.22);
\addplot+[xbar, bar shift=0pt, mark=none, fill=hydraSAMLight, draw=hydraSAM, line width=0.35pt] coordinates {(14.16,10)};
\draw[hydraSAM, line width=0.8pt] (axis cs:14.16,9.78) -- (axis cs:14.16,10.22);
\addplot+[xbar, bar shift=0pt, mark=none, fill=hydraSAMLight, draw=hydraSAM, line width=0.35pt] coordinates {(22.39,9)};
\draw[hydraSAM, line width=0.8pt] (axis cs:22.39,8.78) -- (axis cs:22.39,9.22);
\addplot+[xbar, bar shift=0pt, mark=none, fill=hydraM2FLight, draw=hydraM2F, line width=0.35pt] coordinates {(13.63,8)};
\draw[hydraM2F, line width=0.8pt] (axis cs:13.63,7.78) -- (axis cs:13.63,8.22);
\addplot+[xbar, bar shift=0pt, mark=none, fill=hydraM2FLight, draw=hydraM2F, line width=0.35pt] coordinates {(2.00,7)};
\draw[hydraM2F, line width=0.8pt] (axis cs:2.00,6.78) -- (axis cs:2.00,7.22);
\addplot+[xbar, bar shift=0pt, mark=none, fill=hydraM2FLight, draw=hydraM2F, line width=0.35pt] coordinates {(12.53,6)};
\draw[hydraM2F, line width=0.8pt] (axis cs:12.53,5.78) -- (axis cs:12.53,6.22);
\addplot+[xbar, bar shift=0pt, mark=none, fill=hydraMDINOLight, draw=hydraMDINO, line width=0.35pt] coordinates {(10.47,5)};
\draw[hydraMDINO, line width=0.8pt] (axis cs:10.47,4.78) -- (axis cs:10.47,5.22);
\addplot+[xbar, bar shift=0pt, mark=none, fill=hydraMDINOLight, draw=hydraMDINO, line width=0.35pt] coordinates {(2.08,4)};
\draw[hydraMDINO, line width=0.8pt] (axis cs:2.08,3.78) -- (axis cs:2.08,4.22);
\addplot+[xbar, bar shift=0pt, mark=none, fill=hydraMDINOLight, draw=hydraMDINO, line width=0.35pt] coordinates {(11.38,3)};
\draw[hydraMDINO, line width=0.8pt] (axis cs:11.38,2.78) -- (axis cs:11.38,3.22);
\addplot+[xbar, bar shift=0pt, mark=none, fill=hydraOFLight, draw=hydraOF, line width=0.35pt] coordinates {(17.25,2)};
\draw[hydraOF, line width=0.8pt] (axis cs:17.25,1.78) -- (axis cs:17.25,2.22);
\addplot+[xbar, bar shift=0pt, mark=none, fill=hydraOFLight, draw=hydraOF, line width=0.35pt] coordinates {(2.77,1)};
\draw[hydraOF, line width=0.8pt] (axis cs:2.77,0.78) -- (axis cs:2.77,1.22);
\addplot+[xbar, bar shift=0pt, mark=none, fill=hydraOFLight, draw=hydraOF, line width=0.35pt] coordinates {(13.07,0)};
\draw[hydraOF, line width=0.8pt] (axis cs:13.07,-0.22) -- (axis cs:13.07,0.22);
\draw[hydraGrid, dotted, line width=0.7pt] (axis cs:0,8.50) -- (axis cs:30.5,8.50);
\draw[hydraGrid, dotted, line width=0.7pt] (axis cs:0,5.50) -- (axis cs:30.5,5.50);
\draw[hydraGrid, dotted, line width=0.7pt] (axis cs:0,2.50) -- (axis cs:30.5,2.50);
\end{axis}
\end{tikzpicture}}
    \vspace{-0.75em}
    \captionsetup{width=0.96\linewidth,skip=2pt}
    \caption{\textbf{Hidden top-1 headroom.}
    Conservative oracle gap over frozen candidates. SAM~3 uses class-macro prompt IoU;
    class-aware segmenters use dataset-level semantic mIoU.}
    \label{fig:oracle_gap_analysis}
    \vspace{-0.8em}
\end{wrapfigure}

Before introducing a selector, we ask a diagnostic question: if ground
truth could choose one already-computed candidate from a frozen segmenter,
how much would the output improve? This oracle is not a deployable
method; it is a test of the final selection step. A positive gap means a
better mask was already present in the query pool, so part of the error
lies in choosing the output rather than in missing representation.

\begin{insight}
\textbf{Oracle insight.} A positive top-1 oracle gap is not a new-mask
claim; it means the frozen candidate pool already contains a better
candidate than the one exposed by the default head.
\end{insight}

Let $\hat{y}_0$ be the default prediction for input $x$ and concept
$c$, and let $\mathcal{A}(x,c)$ be the candidate actions available under
the frozen decoder. For promptable SAM~3, an action may keep the default
mask, replace it with one candidate, or add one residual candidate; for
class-aware models it exposes the class-matched query under the same
frozen mask head. The conservative top-1 oracle is
\begin{equation}
    a^\star(x,c)
    \in
    \arg\max_{a\in\mathcal{A}(x,c)}
    \mathrm{IoU}\!\left(\mathcal{T}_a(\hat{y}_0), y\right),
    \qquad
    U^\star(x,c)
    =
    \max_{a\in\mathcal{A}(x,c)}
    \mathrm{IoU}\!\left(\mathcal{T}_a(\hat{y}_0), y\right).
    \label{eq:top1_oracle}
\end{equation}
At the dataset level, we report the oracle gap as
$\miou_{\star}-\miou_{0}$: class-macro prompt IoU for SAM~3, and
dataset-level semantic mIoU for class-aware segmenters.

\FloatBarrier
\addtocounter{figure}{1} 
\begin{figure}[t]
    \centering
    \includegraphics[width=\linewidth]{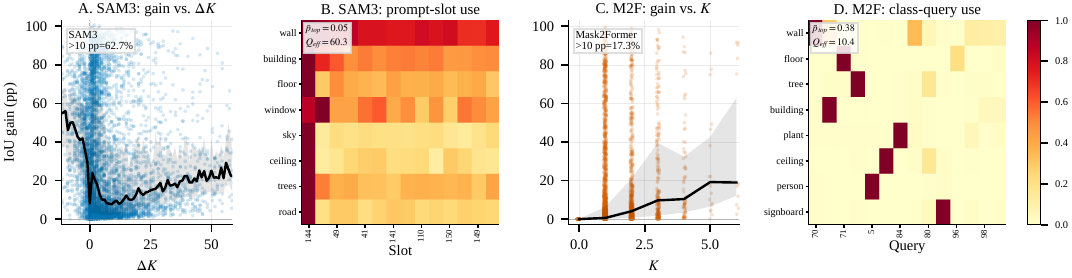}
    \vspace{-0.7em}
    \caption{\textbf{The recoverable signal is structured.}
    Oracle gains arise from small selection changes over frozen candidates:
    changed mask count for SAM~3 and greedy query count for Mask2Former.
    Row-normalized heatmaps show that frequent prompts or classes reuse
    particular mask/query slots rather than arbitrary decoder outputs.}
    \label{fig:query_structure_main}
    \vspace{-0.8em}
\end{figure}

\addtocounter{figure}{-2}

Hidden headroom is not a SAM-specific artifact
(Figure~\ref{fig:oracle_gap_analysis}). Promptable SAM~3 leaves
persistent single-action oracle gaps across broad vocabularies, urban
scenes, driving videos, aerial imagery, and ImageNet-derived prompts.
The class-aware models show the same pattern under dataset-level
semantic mIoU: ADE20k and COCO retain large headroom among frozen
candidates, while Cityscapes is closer to saturated but still nonzero.
Because these are oracle measurements, they show only that the model
often computes a useful alternative before the deployed selection rule
suppresses it.

The gap also cannot be closed by simply trusting the scalar head. In the
compact score-quality diagnostic
(Figure~\ref{fig:score_quality_gap_main}), SAM~3 and Mask2Former both
show only weak alignment between the model's own query score and the IoU
of the corresponding mask; the full four-model diagnostic is deferred to
Appendix~\ref{fig:score_quality_gap_appendix}. Large headroom remains
precisely where the head score gives an overconfident or
under-informative ranking, so the useful signal must be read from the
query output itself. Appendix~\ref{fig:selection_vs_improvement_appendix}
shows that these gains often require only small query-count changes, and
Appendix~\ref{fig:class_query_specialization_appendix} shows the
complementary source of structure in class-aware models: frequent classes
are repeatedly tied to particular query slots. The same appendix includes
SAM~3 as a promptable contrast, where ADE-847 labels spread over many more
prompt-conditioned mask slots.

\begin{wrapfigure}[15]{r}{0.50\textwidth}
    \vspace{-0.8em}
    \centering
    \includegraphics[width=\linewidth]{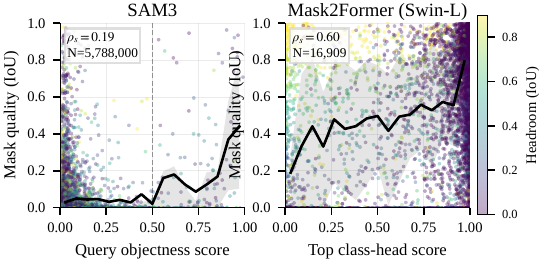}
    \vspace{-1.0em}
    \caption{\textbf{Scores do not rank mask quality.}
    SAM~3 and Mask2Former both show weak alignment between scalar score
    and candidate-mask IoU.}
    \label{fig:score_quality_gap_main}
    \vspace{-0.7em}
\end{wrapfigure}

\addtocounter{figure}{1}

This recoverable signal is structured rather than arbitrary
(Figure~\ref{fig:query_structure_main}). Unlike class-aware segmenters,
SAM~3's default output is already a thresholded multi-mask set, so its
panel plots the oracle's change in selected mask count relative to the
baseline. The different shape is expected: SAM~3
gains can come from adding missed candidates, pruning over-selected
candidates, or replacing candidates within the existing frozen mask pool.
For Mask2Former, by contrast, the oracle chooses among class-matched
query slots, and a few candidate queries can recover substantial IoU. At
the same time, the query slots are not interchangeable. Frequent SAM~3
prompts and common Mask2Former semantic classes repeatedly activate
particular slots, providing the query-use regularity that a post-hoc
selector can learn from frozen outputs.

\section{\ours{}: A Residual Query Selector}
\label{sec:method}

\ours{} treats query selection as conservative output correction. Given
a frozen query pool, the selector either keeps the released output or
executes one residual action using an already-computed candidate. The
backbone, mask head, and default decoder are never updated, and no extra
backbone forward pass is performed.

\subsection{Problem Setup}
\label{sec:method_setup}

A frozen segmenter $f_\phi$ with $Q$ queries emits, for input $x$ and
concept $c$:
\begin{equation}
    f_\phi(x,c)
    =
    \bigl(p_q,\;\alpha_q,\;\beta_q\bigr)_{q=1}^{Q},
    \label{eq:backbone}
\end{equation}
where $p_q$ is the query mask probability, $\alpha_q$ its aggregation
signal, and $\beta_q$ optional auxiliaries. The released aggregator
$\mathcal{G}_0$ maps these outputs to a default prediction
$\hat{y}_0$, and \textbf{we leave $\mathcal{G}_0$ untouched}. The action
set $\mathcal{A}(x,c)$ is derived only from these cached queries. The
training target is the top-1 oracle action in Eq.~\eqref{eq:top1_oracle};
deployment replaces the oracle with a calibrated selector.

\begin{insight}
\textbf{Action semantics.} For SAM~3, the keep-baseline action returns
$\hat{y}_0$, while residual action $q$ exposes the cached alternative
$\mathcal{T}_q(\hat{y}_0)=\hat{y}_0\cup m_q$. For universal segmenters,
action $q$ selects one frozen class-matched query output under the
original mask head. Thus \ours{} changes only the deployed selection over
cached candidates; \textbf{it is not a new mask generator, second
forward pass, or weight update.}
\end{insight}

\subsection{Selector and Training}
\label{sec:architecture}

From cached outputs we build per-query features
$\Phi=(\phi_1,\ldots,\phi_Q)$ and global context $g$. $\Phi$ contains
query statistics such as objectness, geometry, logit moments, agreement,
and exclusivity; $g$ contains image/prompt-level confidence and entropy
signals. The selector is a residual MLP with slot embeddings and bilinear
context gating:
\begin{align}
    \mathbf{h}_q^{(0)} &= \mathrm{LN}(W_\Phi\phi_q+W_g g+W_s e_q),
    \qquad
    \mathbf{h}_q^{(\ell+1)}
    =\mathbf{h}_q^{(\ell)}+\mathrm{MLP}_\ell(\mathrm{LN}(\mathbf{h}_q^{(\ell)})),
    \nonumber\\
    s_q=f_\theta(\Phi,g)_q
    &=w_\mathrm{out}^\top\mathrm{GELU}\!\left(
      W_h\mathrm{LN}\!\left[\mathbf{h}_q^{(L)}
      \,\|\,\mathbf{h}_q^{(L)}\odot W_cg\right]\right).
    \label{eq:selector}
\end{align}
Scores are query-independent given $(\phi_q,g,e_q)$, so selector cost is
linear in $Q$. We use the same architecture family for SAM~3 and the
universal segmenters; only the deterministic feature schema changes with
the backbone. With $h{=}256$ and $L{=}8$, this gives a ${\approx}1.6$M
parameter selector for the class-aware DETR schema and a
$9.6$M--$14.5$M selector for the richer SAM~3 schema, depending on the
dataset vocabulary features. These parameters sit entirely after the
frozen model output.

The two schemas differ only in what the frozen model exposes. For
Mask2Former, MaskDINO, and OneFormer, $\phi_q$ summarizes a class-aware
query: mask geometry, class-logit moments, query confidence, and
agreement with the released dense prediction. Candidate actions are
class-matched query outputs evaluated under the original mask head. For
SAM~3, $\phi_q$ additionally includes prompt-conditioned mask evidence,
box geometry, aggregation statistics, and prompt/text features; the
action set contains a keep-baseline action, called no-op below, plus one
residual-add action per cached mask.
In both cases, the selector never receives ground-truth masks at test
time and never changes the frozen query pool.

For each training item $(x_i,c_i,y_i)$, define the gain of action $a$ as
$\mathbf{g}_{i,a}=\mathrm{IoU}(\mathcal{T}_a(\hat{y}_0),y_i)
-\mathrm{IoU}(\hat{y}_0,y_i)$ and let
$a_i^\star=\arg\max_a\mathbf{g}_{i,a}$. SAM~3 keeps all valid prompts and
uses no-op whenever no residual action clears the minimum gain; DETR
runs may filter zero-headroom examples. With invalid actions masked, the
loss is
\begin{equation}
    \mathcal{L}(\theta)
    =
    (1-\lambda_\mathrm{soft})\,\mathcal{L}_\mathrm{CE}
    + \lambda_\mathrm{soft}\,\mathcal{L}_\mathrm{soft}
    + \lambda_\mathrm{rank}\,\mathcal{L}_\mathrm{rank}.
    \label{eq:loss}
\end{equation}
$\mathcal{L}_\mathrm{CE}$ is gain-weighted cross-entropy on
$a^\star$; $\mathcal{L}_\mathrm{soft}$ is a KL term against a
temperature-$T$ distribution over positive-gain queries; and
$\mathcal{L}_\mathrm{rank}$ is a pairwise margin separating positive
from negative gains. Hyperparameters are fixed within each feature
schema; Appendix~\ref{app:protocol} gives the split, calibration, and
optimization details.

\paragraph{Calibration and regularization.}
Three choices keep the learned selector conservative. First, the policy
can expose at most one cached candidate, so it cannot accumulate many
small, poorly calibrated edits. Second, SAM~3 keeps zero-headroom prompts
in training and labels them no-op; class-aware residual MLP runs may
instead filter examples whose best action is below
$\varepsilon_\mathrm{min}$. Third, the final action margin is calibrated
on held-out images: for SAM~3 it compares the best residual action to
no-op, and for dense-safe universal runs it gates low-confidence
class-query updates. If no safe margin exists, the deployed policy falls
back to the released output. Appendix~\ref{app:protocol} gives the
held-out threshold-selection rule.

\subsection{Inference}
\label{sec:inference}

At test time we extract $(\Phi,g)$ from the frozen outputs, compute
$s=f_{\theta^\star}(\Phi,g)$, mask invalid actions, and select one action.
For SAM~3, the top residual action is executed only when its margin over
no-op exceeds the held-out threshold; otherwise the method returns
$\hat{y}_0$. For dense-safe universal runs, the same idea is applied as a
margin gate over class-query outputs: low-margin updates keep the
default dense prediction. The deployed policy is therefore a calibrated
single-action router over cached candidates.

\FloatBarrier

\section{Experimental Evaluation}
\label{sec:experiments}

\subsection{Setup}
\label{sec:experimental_setup}

We evaluate \ours{} in a frozen-output protocol. Each released model is
run once and cached; \ours{} is trained only on training-split features,
its abstention margin is calibrated on held-out images, and final numbers
are reported once on a disjoint evaluation split. Evaluation labels never
choose features, hyperparameters, margins, or checkpoints; backbone
weights stay fixed except in the marked LoRA control. Universal
segmenters use the original Mask2Former, MaskDINO, and OneFormer
checkpoints~\citep{cheng2022mask2former,li2023maskdino,
jain2023oneformer} on ADE20k, Cityscapes, and COCO; SAM~3 uses
image--class prompts. We call each released inference output the
\emph{baseline}; for SAM~3 this is the public Hugging Face default output.

Metrics follow each output contract. Universal segmenters are scored by
dataset-level semantic mIoU with the released ignore-label conventions;
SAM~3 prompts are scored by binary IoU against the class mask and
macro-averaged over classes. Gains are percentage points over the
baseline; oracle columns use the best single frozen action; recovery is the
fraction of oracle headroom recovered. Full counts, splits, and metric
details are in Appendices~\ref{app:protocol}, \ref{app:sam3_metric},
and~\ref{app:full_results}.

\subsection{Results}
\label{sec:experimental_results}

\paragraph{Universal segmenters.}
\begin{wraptable}[13]{r}{0.54\textwidth}
\centering
\vspace{-0.55em}
\caption{\textbf{Frozen-output routing: universal segmenters.}}
\label{tab:universal_ds_main}
\vspace{-0.65em}
\tiny
\setlength{\tabcolsep}{1.8pt}
\renewcommand{\arraystretch}{0.80}
\resizebox{\linewidth}{!}{%
\begin{tabular}{@{}l l r r r r r@{}}
\specialrule{0.1em}{0pt}{0pt}
\rowcolor{headerblue}
\textbf{Model} & \textbf{Dataset} & \textbf{Default}
& \textbf{\ours{}} & \textbf{$\Delta$ pp} & \textbf{Oracle}
& \textbf{Rec.} \\
\midrule
Mask2Former & ADE20k & 0.5620 & \textbf{0.6361} & \textbf{+7.41} & 0.6983 & 54.4\% \\
\rowcolor{tableblue}
Mask2Former & Cityscapes & 0.8289 & \textbf{0.8364} & \textbf{+0.75} & 0.8489 & 37.5\% \\
Mask2Former & COCO & 0.6729 & \textbf{0.7214} & \textbf{+4.85} & 0.7982 & 38.7\% \\
\rowcolor{sectionblue}
\textbf{Mask2Former} & \textbf{Average} & 0.6879 & \textbf{0.7313} & \textbf{+4.34} & 0.7818 & 43.5\% \\
\midrule
\rowcolor{tableblue}
MaskDINO & ADE20k & 0.4878 & \textbf{0.5414} & \textbf{+5.36} & 0.5925 & 51.2\% \\
MaskDINO & Cityscapes & 0.7978 & \textbf{0.7978} & \textbf{+0.01} & 0.8185 & 0.3\% \\
\rowcolor{tableblue}
MaskDINO & COCO & 0.6056 & \textbf{0.6409} & \textbf{+3.53} & 0.7194 & 31.0\% \\
\rowcolor{sectionblue}
\textbf{MaskDINO} & \textbf{Average} & 0.6304 & \textbf{0.6600} & \textbf{+2.97} & 0.7101 & 27.5\% \\
\midrule
OneFormer & ADE20k & 0.5638 & \textbf{0.6170} & \textbf{+5.31} & 0.7363 & 30.8\% \\
\rowcolor{tableblue}
OneFormer & Cityscapes & 0.8275 & \textbf{0.8282} & \textbf{+0.07} & 0.8552 & 2.6\% \\
OneFormer & COCO & 0.6679 & \textbf{0.7024} & \textbf{+3.45} & 0.7986 & 26.4\% \\
\rowcolor{sectionblue}
\textbf{OneFormer} & \textbf{Average} & 0.6864 & \textbf{0.7159} & \textbf{+2.95} & 0.7967 & 20.0\% \\
\specialrule{0.1em}{0pt}{0pt}
\end{tabular}
}
\vspace{-0.85em}
\end{wraptable}

Table~\ref{tab:universal_ds_main} shows that difficulty is driven more by
domain than by decoder family. ADE20k mixes objects, stuff, and parts,
leaving all three released heads unresolved mask choices and
$+5.31$--$+7.41$ pp gains. COCO is closer to natural-image pretraining but
crowded, giving consistent $+3.45$--$+4.85$ pp gains. Cityscapes has a
narrow road-scene taxonomy that is already well matched, so MaskDINO and
OneFormer have almost no safe headroom. Across backbones, \ours{} exposes
better already-computed candidates when they exist and abstains when the
baseline is safest.

\paragraph{SAM~3 prompt routing.}
\begin{wraptable}[11]{r}{0.52\textwidth}
\centering
\vspace{-0.6em}
\caption{\textbf{Frozen-output routing: SAM~3.}}
\label{tab:sam3_main}
\vspace{-0.65em}
\tiny
\setlength{\tabcolsep}{1.7pt}
\renewcommand{\arraystretch}{0.78}
\resizebox{\linewidth}{!}{%
\begin{tabular}{@{}l r r r r r r@{}}
\specialrule{0.1em}{0pt}{0pt}
\rowcolor{headerblue}
\textbf{Dataset} & \textbf{Default} & \textbf{\ours{}}
& \textbf{$\Delta$ pp} & \textbf{Gap} & \textbf{Rec.} & \textbf{Keep def.} \\
\midrule
ADE-847 & 0.3555 & \textbf{0.4428} & \textbf{+8.72} & +26.82 & 32.5\% & 86.6\% \\
\rowcolor{tableblue}
VOC & 0.7589 & \textbf{0.7868} & \textbf{+2.79} & +5.21 & 53.5\% & 93.9\% \\
Cityscapes & 0.5630 & \textbf{0.6108} & \textbf{+4.79} & +13.00 & 36.8\% & 88.6\% \\
\rowcolor{tableblue}
CamVid & 0.3411 & \textbf{0.4907} & \textbf{+14.96} & +28.95 & 51.7\% & 63.0\% \\
LoveDA & 0.2756 & \textbf{0.3904} & \textbf{+11.48} & +29.15 & 39.4\% & 61.8\% \\
\rowcolor{tableblue}
PC-59 & 0.6084 & \textbf{0.6989} & \textbf{+9.04} & +18.93 & 47.8\% & 80.8\% \\
ImageNet-S50 & 0.7697 & \textbf{0.8805} & \textbf{+11.09} & +14.16 & 78.3\% & 87.2\% \\
\rowcolor{tableblue}
COCO & 0.5102 & \textbf{0.6345} & \textbf{+12.43} & +22.39 & 55.5\% & 55.9\% \\
\midrule
\rowcolor{sectionblue}
\textbf{Avg.} & 0.5228 & \textbf{0.6169} & \textbf{+9.41} & +19.83 & 49.4\% & 77.2\% \\
\specialrule{0.1em}{0pt}{0pt}
\end{tabular}
}
\vspace{-0.95em}
\end{wraptable}

Table~\ref{tab:sam3_main} shows the same output-selection bottleneck in a
prompt-conditioned mask pool. ADE-847 has many rare parts and scene-stuff
labels, so the oracle gap is large but the calibrated gate mostly
abstains; VOC and Cityscapes are more saturated. CamVid, ImageNet-S50, and
COCO expose more rankable alternatives. LoveDA is the useful surprise:
remote-sensing land-cover is not an obvious match to SAM~3's public
image/video pretraining description, yet its reusable geometric regions
make better mask choices learnable. Difficulty is therefore not just
semantic distance from pretraining; it depends on whether the frozen pool
contains reusable candidates and whether held-out calibration can identify
when to expose them.

\paragraph{Where gains concentrate.}
\begin{figure}[t!]
\centering
\vspace{-0.8em}
\includegraphics[width=\linewidth]{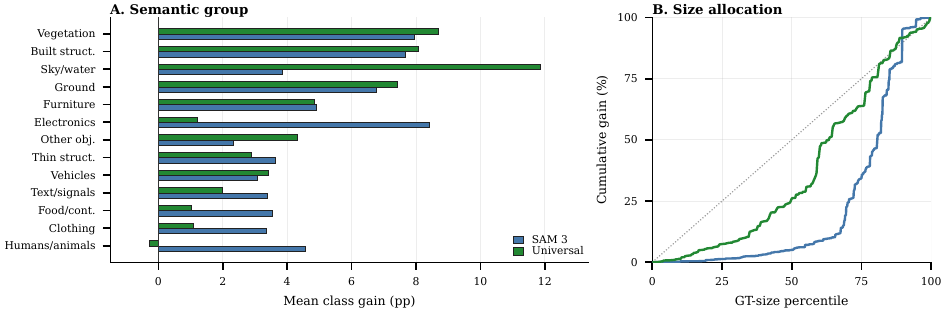}
\vspace{-0.75em}
\caption{\textbf{Where gains concentrate across model families.}
A: mean class gain by coarse semantic group for SAM~3 and the universal
segmenters. B: cumulative positive gain mass after ordering labels from
smallest to largest measured GT masks.}
\label{fig:sam3_label_type_benefit}
\vspace{-1.05em}
\end{figure}

Figure~\ref{fig:sam3_label_type_benefit} separates the allocation by
semantics and by mask size. Panel A shows that routing gains are not
uniformly distributed over labels. Both families benefit on reusable
scene regions such as vegetation, built structures, ground, and
sky/water, but the emphasis differs: universal segmenters gain most on
stuff-like regions, whereas SAM~3 also recovers object-centric categories
such as electronics, furniture, and humans/animals. Thin structures,
vehicles, text/signals, clothing, and food/container labels contribute
less, suggesting that small or highly structured objects leave less
stable unused-candidate headroom. Panel B makes the size effect explicit. Universal
gains accumulate earlier across the size-ranked labels, while SAM~3 gains
are delayed until the upper percentiles, so its largest improvements are
concentrated on larger masks rather than tiny regions.

\paragraph{Slot-use concentration.}
\begin{wrapfigure}[9]{r}{0.45\textwidth}
    \centering
    \vspace{-0.55em}
    \captionsetup{font=footnotesize}
    \includegraphics[width=0.98\linewidth]{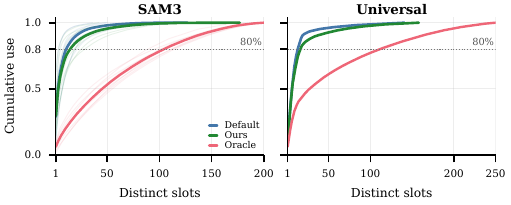}
    \vspace{-0.75em}
    \caption{\textbf{Slot concentration.}
    Default, router, and oracle slot use for SAM~3 and a universal
    segmenter.}
    \label{fig:query_concentration_main}
    \vspace{-0.65em}
\end{wrapfigure}

The query-use diagnostic in
Figure~\ref{fig:query_concentration_main} isolates the selection change.
SAM~3 concentrates the default on a small subset of slots, while the
router draws from a broader frozen pool when it acts. For universal
segmenters, the default curve attributes the dense prediction to the slot
that best overlaps each default class region; on saturated
OneFormer--Cityscapes, \ours{} stays close to that default while the
oracle exposes broader unused headroom. Qualitative before/after
corrections, including nonempty defaults, and per-dataset curves are in
Appendix Figures~\ref{fig:app_qualitative_improvement_row}
and~\ref{fig:app_sam3_query_usage}.

\paragraph{Transfer across domains.}
\begin{wrapfigure}[12]{r}{0.3\textwidth}
    \centering
    \vspace{-0.55em}
    \captionsetup{font=footnotesize}
    \includegraphics[width=1\linewidth]{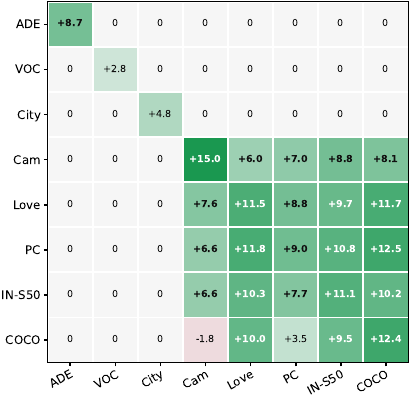}
    \vspace{-0.70em}
    \caption{\textbf{Transfer matrix.}}
    \label{fig:sam3_cross_domain_heatmap}
    \vspace{-0.35em}
\end{wrapfigure}

Cross-domain transfer is broad rather than memorized
(Figure~\ref{fig:sam3_cross_domain_heatmap}). We train a selector on
each source and evaluate it on every target, so the matrix tests whether
learned action semantics survive dataset shift rather than only fitting
diagonal targets. LoveDA, ImageNet-S50, PC-59, and CamVid are strong transfer sources,
suggesting that reusable region geometry transfers even when semantics
change. The asymmetry is informative: COCO$\to$CamVid is the only
remaining negative cell, consistent with a crowded object/stuff source
miscalibrating a smaller driving-video target; Appendix~\ref{app:transfer_controls}
tests vocabulary-size and label-overlap explanations, and
Appendix~\ref{app:coco_camvid_negative_transfer} isolates this negative
cell. Zero bands are meaningful too: for ADE-847, VOC, and Cityscapes
off-domain runs, the no-op margin keeps the default output on every
target prompt.

\paragraph{Safety and model capacity.}
\begin{wrapfigure}[15]{r}{0.50\textwidth}
\vspace{-0.9em}
\centering
\includegraphics[width=1\linewidth]{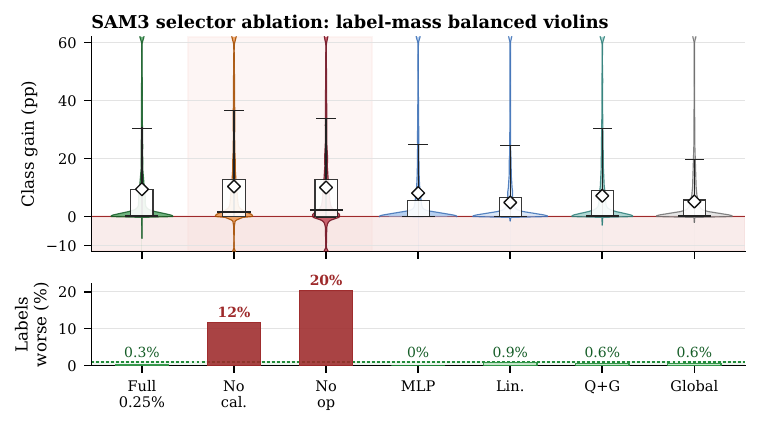}
\vspace{-0.8em}
\caption{\textbf{Selector ablations.}
Label-mass-balanced gains; bottom bars show labels worse than the
default.}
\label{fig:main_ablations}
\vspace{-1.3em}
\end{wrapfigure}

\ours{} is conservative by construction: it can keep the baseline output,
acts at most once, and uses held-out calibration to gate low-margin
replacements. Figure~\ref{fig:main_ablations} audits these choices on
SAM~3 prompt routing and also varies selector capacity and feature inputs;
online throughput is summarized in Figure~\ref{fig:selector_overhead_main_wrap}
and Appendix Figure~\ref{fig:app_selector_overhead_details}. With a fixed
$0.25\%$ validation harm budget, the full safe router gains $+9.4$ class-macro
prompt-IoU points while keeping harmful replacements--actions that lower
prompt IoU relative to the baseline--below $1\%$ on every SAM~3 dataset
($0.30\%$ average). Removing calibration or the no-op increases raw gain to
$+10.4$ and $+10.0$ points, but harmful replacements rise to $15.5\%$ and
$21.3\%$ and worst-quartile gain falls to $-2.38$ and $-2.83$ pp. Simpler
heads and reduced feature sets remain positive, indicating signal beyond one
head choice; Appendix~\ref{app:ablations} gives full safety details.


\paragraph{Online cost.}
\label{sec:online_cost}

\begin{wrapfigure}[8]{r}{0.52\textwidth}
\centering
\resizebox{\linewidth}{!}{\definecolor{hydraBlue}{HTML}{4477AA}
\definecolor{hydraGreen}{HTML}{228833}
\definecolor{hydraModel}{HTML}{D9D9D9}
\definecolor{hydraGrid}{HTML}{E2E2E2}
\definecolor{hydraText}{HTML}{222222}

\begin{tikzpicture}
\pgfplotsset{
  hydra compact axis/.style={
    axis x line*=bottom,
    axis y line*=left,
    axis line style={draw=black, line width=0.65pt},
    tick style={draw=black, line width=0.65pt},
    major tick length=1.8pt,
    tick label style={font=\scriptsize, text=black},
    label style={font=\scriptsize, text=black},
    title style={font=\bfseries\small, text=black, yshift=1pt},
    grid=major,
    grid style={draw=hydraGrid, line width=0.35pt},
    ymajorgrids=true,
    xmajorgrids=false,
    clip=false,
  }
}

\begin{groupplot}[
  group style={
    group size=2 by 1,
    horizontal sep=0.25cm
  },
]

\nextgroupplot[
  hydra compact axis,
  width=4.82cm,
  height=3.02cm,
  title={(a) Live inference},
  ylabel={Latency (ms)},
  ymin=0,
  ymax=270,
  xmin=-0.55,
  xmax=3.55,
  ytick={0,75,150,225},
  xtick={0,1,2,3},
  xticklabels={SAM3,M2F,OF,MD},
  xticklabel style={font=\scriptsize, align=center, text=black},
]

\path[fill=hydraModel, draw=black, line width=0.35pt] (axis cs:-0.18,0) rectangle (axis cs:0.18,214.17);
\path[fill=hydraGreen, draw=black, line width=0.35pt] (axis cs:-0.18,214.17) rectangle (axis cs:0.18,242.58);
\path[fill=hydraModel, draw=black, line width=0.35pt] (axis cs:0.82,0) rectangle (axis cs:1.18,86.12);
\path[fill=hydraGreen, draw=black, line width=0.35pt] (axis cs:0.82,86.12) rectangle (axis cs:1.18,94.42);
\path[fill=hydraModel, draw=black, line width=0.35pt] (axis cs:1.82,0) rectangle (axis cs:2.18,103.22);
\path[fill=hydraGreen, draw=black, line width=0.35pt] (axis cs:1.82,103.22) rectangle (axis cs:2.18,112.76);
\path[fill=hydraModel, draw=black, line width=0.35pt] (axis cs:2.82,0) rectangle (axis cs:3.18,43.69);
\path[fill=hydraGreen, draw=black, line width=0.35pt] (axis cs:2.82,43.69) rectangle (axis cs:3.18,52.04);

\draw[hydraBlue, line width=1.0pt] (axis cs:-0.27,214.47) -- (axis cs:0.27,214.47);
\draw[hydraBlue, line width=1.0pt] (axis cs:0.73,86.12) -- (axis cs:1.27,86.12);
\draw[hydraBlue, line width=1.0pt] (axis cs:1.73,103.22) -- (axis cs:2.27,103.22);
\draw[hydraBlue, line width=1.0pt] (axis cs:2.73,43.69) -- (axis cs:3.27,43.69);

\nextgroupplot[
  hydra compact axis,
  axis y line*=right,
  width=5.28cm,
  height=3.02cm,
  title={(b) Online throughput},
  ylabel={Inputs/s},
  ylabel style={rotate=-90},
  ymin=0,
  ymax=26.5,
  ytick={0,10,20},
  xtick={0,1,2,3},
  xticklabels={SAM3,M2F,OF,MD},
  xticklabel style={font=\scriptsize, align=center, text=black},
  xmin=-0.55,
  xmax=3.55,
]

\addplot+[
  ybar,
  mark=none,
  bar width=5.4pt,
  bar shift=-3.5pt,
  fill=hydraBlue,
  draw=black,
  line width=0.35pt,
] coordinates {
  (0,4.663)
  (1,11.612)
  (2,9.688)
  (3,22.889)
};

\addplot+[
  ybar,
  mark=none,
  bar width=5.4pt,
  bar shift=3.5pt,
  fill=hydraGreen,
  draw=black,
  line width=0.35pt,
] coordinates {
  (0,4.122)
  (1,10.591)
  (2,8.869)
  (3,19.216)
};

\end{groupplot}
\end{tikzpicture}}
\vspace{-0.75em}
\caption{\textbf{Online overhead.}
Blue marks native inference; green is \ours{} with the selector tail.}
\label{fig:selector_overhead_main_wrap}
\vspace{-1.15em}
\end{wrapfigure}
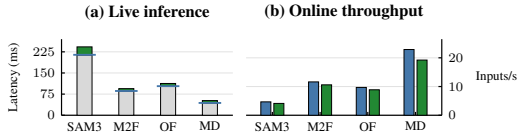

\ours{} reuses the frozen forward rather than calling a second model.
Figure~\ref{fig:selector_overhead_main_wrap} shows routed execution
remains close to native: $4.12$ prompts/s for SAM~3 and $10.59$,
$8.87$, and $19.22$ images/s for Mask2Former, OneFormer, and MaskDINO.
The selector tail adds $28.4$ ms for SAM~3 and $8.3$--$9.5$ ms for the
universal backbones; the full breakdown is in Appendix
Figure~\ref{fig:app_selector_overhead_details}.

\begin{wraptable}[8]{r}{0.42\textwidth}
\vspace{-0.9em}
\centering
\captionsetup{type=table,hypcap=false}
\caption{\textbf{LoRA adaptation control.}}
\label{tab:main_lora}
\vspace{-0.45em}
\scriptsize
\setlength{\tabcolsep}{3pt}
\renewcommand{\arraystretch}{1.02}
\resizebox{\linewidth}{!}{%
\begin{tabular}{@{}lrrrr@{}}
\specialrule{0.1em}{0pt}{0pt}
\rowcolor{headerblue}
\textbf{Backbone} &
\textbf{LoRA} &
\textbf{\ours{}} &
\textbf{\ours{}+L} &
\textbf{Change} \\
\midrule
Mask2Former & $-0.57$ & $+7.21$ & \textbf{$+7.59$} & $+0.37$ \\
\rowcolor{tableblue}
MaskDINO & $-1.52$ & $+6.49$ & \textbf{$+7.20$} & $+0.71$ \\
\midrule
\rowcolor{sectionblue}
\textbf{Average} & $\mathbf{-1.05}$ & $\mathbf{+6.85}$ &
\textbf{$+7.40$} & $\mathbf{+0.54}$ \\
\specialrule{0.1em}{0pt}{0pt}
\end{tabular}
}
\vspace{-0.2em}

{\scriptsize\emph{Change} is \ours{}+LoRA minus \ours{}.}

\vspace{-0.7em}
\end{wraptable}

\paragraph{Weight adaptation baseline.}
LoRA changes weights rather than the final selection rule. If weight
misalignment were sufficient, it should improve exposed masks and reduce
routing headroom.
Instead, LoRA alone changes per-prompt mIoU by $-2.24$ to $+0.24$ points,
with five of six cells negative; retraining \ours{} on the same LoRA cache
changes gain by only $+0.54$ points on average. The decoder still has to
expose a candidate already present; full results are in
Table~\ref{tab:app_lora_results}.

\FloatBarrier

\section{Related Work}
\label{sec:related}

DETR~\citep{carion2020detr} introduced learned queries trained by
bipartite assignment, and Mask2Former, MaskDINO, and OneFormer extend this
recipe to universal segmentation~\citep{cheng2022mask2former,
li2023maskdino,jain2023oneformer}; promptable and open-vocabulary systems
also expose candidate sets~\citep{kirillov2023sam,carion2025sam3,
liu2024groundingdino,zou2023xdecoder}. DETR variants improve query
generation, assignment, or score quality
\citep{li2022dndetr,zhang2023dino,liu2022dabdetr,pu2023rankdetr,
cai2023aligndetr,kang2026paqdetr}, while calibration, mask-quality scoring,
and non-maximum suppression improve exposed predictions
\citep{guo2017calibration,munir2024caldetr,huang2019maskscoring,
bodla2017softnms}. Test-time adaptation~\citep{wang2021tent,
wang2022cotta}, PEFT~\citep{jia2022vpt,hu2022lora,chen2022adaptformer,
lian2022ssf}, and segmentation adapters~\citep{ke2023hqsam,
zhang2024persam,liu2024matcher,fu2024frozendetr} change weights or
inference-time computation. \ours{} instead operates after one frozen
forward, routing only among decoded candidates for the current input; unlike
expert routing~\citep{fedus2022switch,zhou2022expertchoice,
puigcerver2024softmoe}, it does not route among parameterized experts.
Oracle analyses are common in detection and segmentation
\citep{borji2019upperbound,huang2019maskscoring}. Park
\etal{}~\citep{park2024reliable} suppress specialist DETR predictions,
whereas we find that some discarded specialists are useful when routed;
FSSDINO~\citep{zakir2026fssdino} reports a related selection gap in frozen
DINOv3, while we study query-based segmenters and recover accuracy with a
frozen-output selector.


\section{Conclusion}
\label{sec:conclusion}
\vspace{-0.5em}

This paper separates whether a frozen segmenter can compute a good mask
from whether its deployed output selection exposes it. Across DETR-family
universal segmenters and SAM~3 prompt-conditioned mask pools, frozen
outputs often contain better candidates than the default prediction.
\ours{} recovers part of this hidden quality with a small selector trained
only on cached candidates, improving Mask2Former, MaskDINO, OneFormer, and
all eight SAM~3 domains. A calibrated keep-baseline gate is central: less
conservative routers can raise average IoU but worsen the lower tail.

LoRA controls and the TinyDETR study qualify the claim. Routing is not a
substitute for stronger adaptation or better backbones; it identifies a
separable output-selection bottleneck that persists after lightweight
adaptation and is consistent with query specialization under bipartite
matching. The practical takeaway is simple: frozen segmenters should be
evaluated by both the masks they expose and the useful candidates they
suppress.

\newpage
\bibliographystyle{plainnat}
\bibliography{references}

\clearpage
\section*{NeurIPS Paper Checklist}

\begin{enumerate}

\item {\bf Claims}
    \item[] Question: Do the main claims made in the abstract and introduction accurately reflect the paper's contributions and scope?
    \item[] Answer: \answerYes{}.
    \item[] Justification: The abstract and Introduction state the frozen-output, cached-candidate scope, the no-new-mask/no-weight-update assumption, the main empirical gains, and the mechanism claim. The experiments, LoRA controls, and TinyDETR study support those claims in Sections~\ref{sec:theory}--\ref{sec:experiments} and the appendix.

\item {\bf Limitations}
    \item[] Question: Does the paper discuss the limitations of the work performed by the authors?
    \item[] Answer: \answerYes{}.
    \item[] Justification: Appendix~\ref{app:limits} explicitly discusses the supervised oracle-label requirement, the labelled calibration split, and the single-action restriction, and it outlines future work beyond the current scope.

\item {\bf Theory assumptions and proofs}
    \item[] Question: For each theoretical result, does the paper provide the full set of assumptions and a complete (and correct) proof?
    \item[] Answer: \answerYes{}.
    \item[] Justification: The formal claims are stated as propositions in Section~\ref{sec:theory} and Appendix~\ref{app:controlled}; proofs are provided either inline or in Appendix~\ref{app:proofs}.

    \item {\bf Experimental result reproducibility}
    \item[] Question: Does the paper fully disclose all the information needed to reproduce the main experimental results of the paper to the extent that it affects the main claims and/or conclusions of the paper (regardless of whether the code and data are provided or not)?
    \item[] Answer: \answerYes{}.
    \item[] Justification: Sections~\ref{sec:method} and~\ref{sec:experimental_setup} describe the selector and frozen-output protocol, while Appendix~\ref{app:protocol} gives datasets, splits, cached features, optimization hyperparameters, and held-out calibration.

\item {\bf Open access to data and code}
    \item[] Question: Does the paper provide open access to the data and code, with sufficient instructions to faithfully reproduce the main experimental results, as described in supplemental material?
    \item[] Answer: \answerNo{}.
    \item[] Justification: The paper uses public datasets and released model checkpoints and provides protocol details, but an anonymized code release with exact reproduction commands is not included in the current submission.

\item {\bf Experimental setting/details}
    \item[] Question: Does the paper specify all the training and test details (e.g., data splits, hyperparameters, how they were chosen, type of optimizer) necessary to understand the results?
    \item[] Answer: \answerYes{}.
    \item[] Justification: Section~\ref{sec:experimental_setup} gives the evaluation protocol and metrics; Appendix~\ref{app:protocol} gives image-level splitting, cache construction, selector hyperparameters, early stopping, and threshold selection.

\item {\bf Experiment statistical significance}
    \item[] Question: Does the paper report error bars suitably and correctly defined or other appropriate information about the statistical significance of the experiments?
    \item[] Answer: \answerNo{}.
    \item[] Justification: The paper reports deterministic held-out evaluations over fixed image/prompt splits and gives per-dataset results, safety rates, and lower-tail diagnostics, but it does not report confidence intervals or repeated-run error bars.

\item {\bf Experiments compute resources}
    \item[] Question: For each experiment, does the paper provide sufficient information on the computer resources (type of compute workers, memory, time of execution) needed to reproduce the experiments?
    \item[] Answer: \answerNo{}.
    \item[] Justification: The paper reports online latency and throughput in Section~\ref{sec:experiments} and Appendix~\ref{app:plots}, but it does not provide full hardware, memory, or total training-compute requirements for every experiment.
    
\item {\bf Code of ethics}
    \item[] Question: Does the research conducted in the paper conform, in every respect, with the NeurIPS Code of Ethics \url{https://neurips.cc/public/EthicsGuidelines}?
    \item[] Answer: \answerYes{}.
    \item[] Justification: The work evaluates segmentation methods on established public benchmarks and released checkpoints, does not involve human-subject experiments or private data collection, and is presented anonymously.

\item {\bf Broader impacts}
    \item[] Question: Does the paper discuss both potential positive societal impacts and negative societal impacts of the work performed?
    \item[] Answer: \answerNo{}.
    \item[] Justification: The paper focuses on a technical output-selection bottleneck in frozen segmentation models and does not include a separate broader-impact discussion. Potential impacts are indirect and depend on downstream segmentation deployments.
    
\item {\bf Safeguards}
    \item[] Question: Does the paper describe safeguards that have been put in place for responsible release of data or models that have a high risk for misuse (e.g., pre-trained language models, image generators, or scraped datasets)?
    \item[] Answer: \answerNA{}.
    \item[] Justification: The paper does not introduce or release a high-risk generative model, scraped dataset, or new pretrained foundation model; it studies selectors over existing frozen segmentation outputs.

\item {\bf Licenses for existing assets}
    \item[] Question: Are the creators or original owners of assets (e.g., code, data, models), used in the paper, properly credited and are the license and terms of use explicitly mentioned and properly respected?
    \item[] Answer: \answerNo{}.
    \item[] Justification: The paper cites the datasets and released models used in the experiments, but it does not explicitly list every asset's license or terms of use in the current manuscript.

\item {\bf New assets}
    \item[] Question: Are new assets introduced in the paper well documented and is the documentation provided alongside the assets?
    \item[] Answer: \answerNA{}.
    \item[] Justification: The submission does not introduce a new dataset, benchmark, or released pretrained model asset.

\item {\bf Crowdsourcing and research with human subjects}
    \item[] Question: For crowdsourcing experiments and research with human subjects, does the paper include the full text of instructions given to participants and screenshots, if applicable, as well as details about compensation (if any)? 
    \item[] Answer: \answerNA{}.
    \item[] Justification: The work does not involve crowdsourcing or human-subject experiments.

\item {\bf Institutional review board (IRB) approvals or equivalent for research with human subjects}
    \item[] Question: Does the paper describe potential risks incurred by study participants, whether such risks were disclosed to the subjects, and whether Institutional Review Board (IRB) approvals (or an equivalent approval/review based on the requirements of your country or institution) were obtained?
    \item[] Answer: \answerNA{}.
    \item[] Justification: The work does not involve crowdsourcing or human-subject experiments, so IRB approval is not applicable.

\item {\bf Declaration of LLM usage}
    \item[] Question: Does the paper describe the usage of LLMs if it is an important, original, or non-standard component of the core methods in this research? Note that if the LLM is used only for writing, editing, or formatting purposes and does \emph{not} impact the core methodology, scientific rigor, or originality of the research, declaration is not required.
    \item[] Answer: \answerNA{}.
    \item[] Justification: LLMs are not an important, original, or non-standard component of the core method or experiments.

\end{enumerate}

\clearpage
\raggedbottom
\appendix

\setcounter{figure}{0}
\setcounter{table}{0}
\renewcommand{\thefigure}{\thesection.\arabic{figure}}
\renewcommand{\thetable}{\thesection.\arabic{table}}

\section{Appendix Overview and Organization}
\label{app:overview}

The appendix follows the claim structure of the main paper. It first
records the experimental protocol, then gives the full result tables and
representative diagnostics behind the main claims, and finally collects
the controlled mechanism study, formal proofs, and scope notes.

\begin{appendixroadmap}
\textbf{Reading guide.} The appendix is grouped into four blocks:
setup, extended results, mechanistic analysis, and formal support.
\begin{itemize}[leftmargin=1.2em,itemsep=1pt,topsep=2pt]
    \item \textbf{Setup and reproducibility:}
    Appendix~\ref{app:protocol} describes the frozen-output protocol,
    splitting rule, SAM~3 metric convention, and implementation details
    used by all selector runs.
    \item \textbf{Extended empirical results:}
    Appendix~\ref{app:full_results} reports the full SAM~3,
    DETR-family, and LoRA tables with training/evaluation counts, and
    Appendix~\ref{app:ablations} gives compact ablation checks.
    Appendix~\ref{app:plots} organizes the representative query-use,
    recovery, transfer, and qualitative diagnostics into local
    subsections.
    \item \textbf{Mechanistic analysis:}
    Appendix~\ref{app:controlled} gives the controlled TinyDETR study:
    specialization emerges under matching, routing errors concentrate at
    ambiguous distribution boundaries, and query features carry the
    recoverable signal.
    \item \textbf{Formal support and scope:}
    Appendix~\ref{app:proofs} contains deferred proofs and identifiability
    arguments. Appendix~\ref{app:limits} summarizes limitations, future
    directions, and the notation used in the paper.
\end{itemize}
\end{appendixroadmap}

For readability, figures and tables are numbered within each appendix
section.

\section{Protocol and Reproducibility}
\label{app:protocol}

\paragraph{Frozen-output protocol.}
For every backbone and dataset, we first cache the frozen model outputs:
query masks, confidence/class signals, and any exposed query-level
auxiliaries. The selector is trained only on deterministic features
derived from this cache. No selector run performs an additional backbone
forward pass at inference time, and no experiment updates the
segmentation backbone unless it is explicitly marked as a LoRA run.

\paragraph{Benchmarks.}
The universal segmenter runs use ADE20k~\citep{zhou2019ade20k},
Cityscapes~\citep{cordts2016cityscapes}, and COCO/COCO-Stuff
annotations~\citep{lin2014coco,caesar2018cocostuff}. The SAM~3 promptable
suite uses ADE-847~\citep{zhou2019ade20k}, Pascal
VOC~\citep{everingham2015pascalvoc}, Cityscapes, CamVid~\citep{brostow2009camvid},
LoveDA~\citep{wang2021loveda}, Pascal Context-59~\citep{mottaghi2014rolecontext},
ImageNet-S50~\citep{gao2022large}, and COCO/COCO-Stuff.

\paragraph{Splits.}
All selector splits are made at the image level. If multiple prompts or
class instances originate from the same image, they remain in the same
split. This prevents prompt leakage and makes the reported train/eval
counts in Appendix~\ref{app:full_results} the active prompt counts after
filtering, not raw image counts.

\paragraph{Targets and filtering.}
For the universal segmenter runs, the target is the class-matched query
that maximizes held-out prompt IoU under the frozen mask head. For the
corrected SAM~3 runs, the target is the best single action among no-op
and residual-add actions; the learned margin then decides whether to keep
the default output at inference. If no residual action exceeds the
minimum-gain setting, no-op is the target; these prompts remain in the
SAM~3 train and calibration splits. Evaluation is always reported over
the held-out prompt set. For the largest SAM~3 datasets, the training
cap uses metadata-only stratified prompt sampling before any IoU-derived
targets are computed; the eval split remains full.

\paragraph{Optimization.}
The main selector is the bilinear residual MLP described in
Section~\ref{sec:method}. Within a feature schema, hyperparameters are
fixed across datasets:
$\lambda_{\mathrm{soft}}=0.75$, $T=0.10$,
$\lambda_{\mathrm{rank}}=0.20$, and
$\varepsilon_{\min}=10^{-3}$. Early stopping uses validation recovered
gain. For SAM~3, the final no-op margin is calibrated on the held-out
calibration split. For each prompt, we compute the score margin between
the best residual action and the keep-baseline action; a candidate
threshold $\tau$ executes the residual action only when that margin exceeds
$\tau$. We sweep $\tau$ over calibration margins, measure the resulting
prompt-IoU gain and harmful-action rate, and choose the threshold with the
largest validation gain subject to the fixed harmful-action budget
($0.25\%$ in the main runs). If no threshold satisfies the budget, the
policy keeps the baseline output. The class-aware DETR-family schema
yields a selector of approximately 1.6M parameters; the richer SAM~3
residual schema yields roughly 9.6M--14.5M parameters depending on dataset
vocabulary features.

\subsection{Why SAM~3 Uses Class-Macro Prompt IoU}
\label{app:sam3_metric}

SAM~3 and the DETR-family universal segmenters expose different output
contracts. Mask2Former, MaskDINO, and OneFormer produce class-aware query
scores over a fixed label set. Their standard inference rule combines
query masks with class probabilities and then applies an $\arg\max_c$
over the closed vocabulary, so inter-class competition is part of the
model output: each pixel is resolved to one semantic class.

SAM~3 does not perform this resolution. It is prompted with a single
concept at a time and answers whether and where instances of that
concept appear. If the same image is prompted with two related concepts,
for example \emph{person} and \emph{rider}, the two forward passes can
produce masks that overlap arbitrarily. The model has not compared those
concepts, and nothing in its output forces masks from different prompts
to form a semantic partition of the image.

Pixel-pooled semantic mIoU assumes such a partition. Applying it to SAM~3
therefore requires an additional rule for resolving overlapping
prompt-specific masks, such as highest score wins, first prompt wins, or
presence-score weighting. That rule is a post-hoc design choice rather
than part of SAM~3. A pixel-pooled semantic mIoU number would consequently mix
two effects: the quality of SAM~3's prompt-conditioned query pool and the
quality of the external tie-breaker.

Prompt IoU avoids this extra assumption. Each prompt-conditioned forward
pass is evaluated against the corresponding binary ground-truth mask.
For paper tables, we macro-average those prompt IoUs by class: first
average all rows for a prompt label, then average the prompt-label means.
This matches the signal \ours{} acts on for SAM~3---intra-concept query
selection inside a single prompt, not inter-class competition across
prompts---while preventing frequent labels from dominating broad
datasets. We therefore report SAM~3 selector results with class-macro
prompt IoU in the paper tables; row-mean prompt IoU is kept only for
diagnostic plots, oracle-anatomy checks, and the released CSVs.

\section{Full Result Tables}
\label{app:full_results}

\subsection{SAM~3 Promptable Results}
\label{app:sam3_full_results}

\begin{table}[H]
\centering
\tiny
\setlength{\tabcolsep}{2.0pt}
\renewcommand{\arraystretch}{1.10}
\caption{\textbf{SAM~3 no-op query routing.}
All quality numbers are class-macro prompt IoU from the baseline. The
selector scores one residual query action plus an explicit no-op action,
then applies a validation-calibrated no-op margin. The final row
macro-averages datasets.}
\label{tab:app_sam3_results}
\resizebox{\linewidth}{!}{%
\begin{tabular}{@{}l l r r r r r r r r r r@{}}
\specialrule{0.1em}{0pt}{0pt}
\rowcolor{headerblue}
\textbf{Dataset} & \textbf{Run} & \textbf{Train (+val/cal)}
& \textbf{Eval} & \textbf{Default} & \textbf{\ours{}}
& \textbf{$\Delta$ pp} & \textbf{Gap pp} & \textbf{Rec.}
& \textbf{Action} & \textbf{Keep def.} & \textbf{Harmful} \\
\specialrule{0em}{0pt}{0pt}
\midrule
ADE-847 & No-op & 16{,}070 (+3{,}930) & 28{,}940 & 0.3555 & \textbf{0.4428} & \textbf{+8.72} & +26.82 & 32.5\% & 13.4\% & 86.6\% & 0.08\% \\
\rowcolor{tableblue}
VOC & No-op & 1{,}724 (+446) & 2{,}148 & 0.7589 & \textbf{0.7868} & \textbf{+2.79} & +5.21 & 53.5\% & 6.1\% & 93.9\% & 0.51\% \\
Cityscapes & No-op & 27{,}762 (+6{,}961) & 6{,}002 & 0.5630 & \textbf{0.6108} & \textbf{+4.79} & +13.00 & 36.8\% & 11.4\% & 88.6\% & 0.18\% \\
\rowcolor{tableblue}
CamVid & No-op & 4{,}519 (+1{,}138) & 1{,}560 & 0.3411 & \textbf{0.4907} & \textbf{+14.96} & +28.95 & 51.7\% & 37.0\% & 63.0\% & 0.32\% \\
LoveDA & No-op & 7{,}722 (+1{,}951) & 6{,}948 & 0.2756 & \textbf{0.3904} & \textbf{+11.48} & +29.15 & 39.4\% & 38.2\% & 61.8\% & 0.75\% \\
\rowcolor{tableblue}
PC-59 & No-op & 20{,}035 (+4{,}965) & 25{,}974 & 0.6084 & \textbf{0.6989} & \textbf{+9.04} & +18.93 & 47.8\% & 19.2\% & 80.8\% & 0.36\% \\
ImageNet-S50 & No-op & 400 (+100) & 759 & 0.7697 & \textbf{0.8805} & \textbf{+11.09} & +14.16 & 78.3\% & 12.8\% & 87.2\% & 0.00\% \\
\rowcolor{tableblue}
COCO & No-op & 20{,}060 (+4{,}940) & 34{,}727 & 0.5102 & \textbf{0.6345} & \textbf{+12.43} & +22.39 & 55.5\% & 44.1\% & 55.9\% & 0.22\% \\
\midrule
\rowcolor{sectionblue}
\textbf{Average} & -- & -- & -- & 0.5228 & \textbf{0.6169} & \textbf{+9.41} & +19.83 & 49.4\% & 22.8\% & 77.2\% & 0.30\% \\
\specialrule{0.1em}{0pt}{0pt}
\end{tabular}
}
\vspace{-0.25em}
\end{table}

Table~\ref{tab:app_sam3_results} expands the main SAM~3 table with the
full class-macro prompt-IoU selector metrics: train/validation-calibration/eval
prompt counts, default IoU, \ours{} IoU, top-1 action headroom, recovery,
action rate, no-op rate, and harmful-action rate. All rows use the
baseline cache and the audited residual no-op selector. The
main pattern is stable across dataset regimes: every promptable benchmark
improves under class-macro prompt IoU, but the magnitude follows headroom.
CamVid has the largest absolute gain ($+14.96$ pp), followed by
COCO, LoveDA, ImageNet-S50, PC-59, and ADE-847; Cityscapes and VOC gain less because
their default prompt masks are already stronger or the calibrated gate
acts sparsely. Recovery ranges from $32.5\%$ on ADE-847 to $78.3\%$
on ImageNet-S50, so the selector exposes a substantial fraction of the
available single-action headroom without relying on one dataset's query
layout.

The broad-vocabulary rows are important for interpretation. ADE-847 and
COCO both contain many more visual concepts than the smaller driving or
object benchmarks, yet the selector still improves class-macro prompt IoU by
$+8.72$ points on ADE-847 and $+12.43$ points on COCO. This supports the paper's claim
that SAM~3's residual gap is not only a narrow-domain artifact: even when
the prompt space is diverse, the frozen query pool often contains a
better mask than the default selection exposes.

\subsection{SAM~3 Oracle Anatomy}
\label{app:sam3_oracle_anatomy}

\begin{table}[H]
\centering
\scriptsize
\setlength{\tabcolsep}{3.0pt}
\renewcommand{\arraystretch}{1.10}
\caption{\textbf{SAM~3 oracle anatomy.}
All values are prompt-level IoU. These diagnostics use ground truth to
measure frozen SAM~3 candidate headroom; they are not learned selector
results.}
\label{tab:app_sam3_oracle_anatomy}
\begin{tabular}{@{}l r r r r r r r@{}}
\specialrule{0.1em}{0pt}{0pt}
\rowcolor{headerblue}
\textbf{Dataset / setting} & \textbf{Prompts} & \textbf{Default}
& \textbf{Single} & \textbf{Greedy} & \textbf{Single gap}
& \textbf{Greedy gap} & \textbf{Mean K} \\
\midrule
VOC & 2{,}148 & 0.7912 & 0.7625 & 0.8628 & -2.87 & +7.16 & 7.03 \\
\rowcolor{tableblue}
Cityscapes & 6{,}005 & 0.6253 & 0.5989 & 0.8080 & -2.64 & +18.28 & 14.43 \\
LoveDA & 7{,}326 & 0.0520 & 0.2175 & 0.3227 & +16.56 & +27.08 & 19.40 \\
\rowcolor{tableblue}
CamVid & 1{,}562 & 0.4824 & 0.6059 & 0.7564 & +12.36 & +27.40 & 10.07 \\
ADE-847 & 29{,}001 & 0.5296 & 0.6483 & 0.7876 & +11.87 & +25.80 & 11.97 \\
\rowcolor{tableblue}
ADE-847, baseline setting & 29{,}001 & 0.4818 & 0.6456 & 0.7849 & +16.38 & +30.31 & 11.88 \\
\specialrule{0.1em}{0pt}{0pt}
\end{tabular}
\vspace{-0.8em}
\end{table}

Table~\ref{tab:app_sam3_oracle_anatomy} reports the pure SAM~3 oracle
anatomy runs, again using prompt-level IoU diagnostics. These are not learned
selector results: they measure how much headroom exists if one can choose
queries using ground truth. The single-query oracle can be below the
default on saturated VOC and Cityscapes runs because the default mask is
already a multi-query prediction; the greedy oracle shows the larger
multi-query ceiling. The ADE-847 rows compare the cached and baseline
settings, and both show substantial prompt-level
headroom.

\subsection{Universal Segmenter Results}
\label{app:universal_full_results}

\begin{table}[H]
\centering
\scriptsize
\setlength{\tabcolsep}{2.4pt}
\renewcommand{\arraystretch}{1.15}
\caption{\textbf{Frozen DETR-family selectors.}
This appendix table reports the per-prompt routing diagnostic for
Mask2Former, MaskDINO, and OneFormer. The main text reports the
paper-comparable DS mIoU in Table~\ref{tab:universal_ds_main}. Average
rows macro-average the displayed datasets for each model.}
\label{tab:app_detr_results}
\begin{tabular}{@{}l l r r r r r r r r@{}}
\specialrule{0.1em}{0pt}{0pt}
\rowcolor{headerblue}
\textbf{Model} & \textbf{Dataset} & \textbf{Queries} & \textbf{Train}
& \textbf{Eval} & \textbf{Default} & \textbf{\ours{}}
& \textbf{$\Delta$ pp} & \textbf{Top-1 oracle} & \textbf{Recovery} \\
\specialrule{0em}{0pt}{0pt}
\midrule
\rowcolor{sectionblue}
\multicolumn{10}{@{}l}{\textbf{Mask2Former}} \\
Mask2Former & ADE20k & 100 & 47{,}238 & 9{,}103 & 0.6304 & \textbf{0.7295} & \textbf{+9.92} & 0.7689 & 71.6\% \\
\rowcolor{tableblue}
Mask2Former & Cityscapes & 100 & 6{,}612 & 2{,}519 & 0.7328 & \textbf{0.7560} & \textbf{+2.33} & 0.7749 & 55.3\% \\
Mask2Former & COCO & 200 & 85{,}601 & 23{,}575 & 0.6440 & \textbf{0.7379} & \textbf{+9.39} & 0.7765 & 70.9\% \\
\rowcolor{sectionblue}
\textbf{Mask2Former} & \textbf{Average} & -- & -- & -- & 0.6691 & \textbf{0.7411} & \textbf{+7.21} & 0.7734 & 65.9\% \\
\midrule
\rowcolor{sectionblue}
\multicolumn{10}{@{}l}{\textbf{MaskDINO}} \\
MaskDINO & ADE20k & 100 & 47{,}851 & 8{,}081 & 0.6187 & \textbf{0.7041} & \textbf{+8.54} & 0.7414 & 69.6\% \\
\rowcolor{tableblue}
MaskDINO & Cityscapes & 100 & 8{,}433 & 2{,}996 & 0.7349 & \textbf{0.7526} & \textbf{+1.77} & 0.7656 & 57.8\% \\
MaskDINO & COCO & 300 & 81{,}932 & 20{,}837 & 0.6297 & \textbf{0.7176} & \textbf{+8.79} & 0.7585 & 68.2\% \\
\rowcolor{sectionblue}
\textbf{MaskDINO} & \textbf{Average} & -- & -- & -- & 0.6611 & \textbf{0.7248} & \textbf{+6.37} & 0.7552 & 65.2\% \\
\midrule
\rowcolor{sectionblue}
\multicolumn{10}{@{}l}{\textbf{OneFormer}} \\
OneFormer & ADE20k & 250 & 68{,}011 & 10{,}724 & 0.6497 & \textbf{0.7399} & \textbf{+9.02} & 0.7807 & 68.9\% \\
\rowcolor{tableblue}
OneFormer & Cityscapes & 250 & 9{,}032 & 3{,}169 & 0.7818 & \textbf{0.7839} & \textbf{+0.21} & 0.8180 & 5.8\% \\
OneFormer & COCO & 150 & 85{,}014 & 22{,}815 & 0.6258 & \textbf{0.7285} & \textbf{+10.27} & 0.7643 & 74.2\% \\
\rowcolor{sectionblue}
\textbf{OneFormer} & \textbf{Average} & -- & -- & -- & 0.6858 & \textbf{0.7508} & \textbf{+6.50} & 0.7877 & 49.6\% \\
\specialrule{0.1em}{0pt}{0pt}
\end{tabular}
\vspace{-0.8em}
\end{table}

Table~\ref{tab:app_detr_results} gives the per-prompt routing diagnostics
behind the universal-segmenter results in the main paper. The largest
gains occur on ADE20k and COCO, where broad scene/object vocabularies
make query competition more consequential; Cityscapes is more saturated
and leaves less recoverable headroom.

\subsection{LoRA Adaptation Control}
\label{app:lora_control}

\begin{table}[H]
\centering
\scriptsize
\setlength{\tabcolsep}{4.8pt}
\renewcommand{\arraystretch}{1.15}
\caption{\textbf{LoRA as a lightweight adaptation baseline and invariance check.}
Each row fixes the backbone, dataset, and split, then compares LoRA
alone with \ours{} trained on frozen outputs and on LoRA-adapted outputs.
LoRA alone is usually negative, while the \ours{} gain is nearly
unchanged after LoRA. Average rows macro-average datasets within each
backbone.}
\label{tab:app_lora_results}
\begin{tabular}{@{}l l r r r r@{}}
\specialrule{0.1em}{0pt}{0pt}
\rowcolor{headerblue}
\textbf{Backbone} & \textbf{Dataset} & \textbf{LoRA $\Delta$}
& \textbf{\ours{} $\Delta$ frozen}
& \textbf{\ours{} $\Delta$ after LoRA}
& \textbf{Change} \\
\specialrule{0em}{0pt}{0pt}
\midrule
\rowcolor{sectionblue}
\multicolumn{6}{@{}l}{\textbf{Mask2Former}} \\
Mask2Former & ADE20k & $-0.52$ & \textbf{$+9.92$} & \textbf{$+9.85$} & $-0.07$ \\
\rowcolor{tableblue}
Mask2Former & Cityscapes & $+0.24$ & $+2.33$ & $+2.43$ & $+0.10$ \\
Mask2Former & COCO & $-1.44$ & \textbf{$+9.39$} & \textbf{$+10.48$} & $+1.09$ \\
\rowcolor{sectionblue}
\textbf{Mask2Former} & \textbf{Average} & $\mathbf{-0.57}$ & \textbf{$+7.21$} & \textbf{$+7.59$} & $\mathbf{+0.37}$ \\
\midrule
\rowcolor{sectionblue}
\multicolumn{6}{@{}l}{\textbf{MaskDINO}} \\
MaskDINO & ADE20k & $-1.27$ & \textbf{$+8.79$} & \textbf{$+8.96$} & $+0.17$ \\
\rowcolor{tableblue}
MaskDINO & Cityscapes & $-1.05$ & $+1.80$ & $+1.98$ & $+0.18$ \\
MaskDINO & COCO & $-2.24$ & \textbf{$+8.87$} & \textbf{$+10.67$} & $+1.80$ \\
\rowcolor{sectionblue}
\textbf{MaskDINO} & \textbf{Average} & $\mathbf{-1.52}$ & \textbf{$+6.49$} & \textbf{$+7.20$} & $\mathbf{+0.71}$ \\
\midrule
\rowcolor{tableblue}
\textbf{Average} & -- & $-1.05$ & \textbf{$+6.85$} & \textbf{$+7.40$} & $+0.54$ \\
\specialrule{0.1em}{0pt}{0pt}
\end{tabular}
\vspace{-0.8em}
\end{table}

Table~\ref{tab:app_lora_results} tests whether the observed gains are
explained by lightweight backbone adaptation. LoRA alone is often
negative, while the frozen-output selector remains positive before and
after LoRA, indicating that the main effect is output selection over
the candidate pool rather than generic fine-tuning.

\section{Ablation and Robustness Checks}
\label{app:ablations}

\begin{table}[H]
\centering
\scriptsize
\setlength{\tabcolsep}{2.1pt}
\renewcommand{\arraystretch}{1.08}
\caption{\textbf{SAM~3 selector ablations across promptable datasets.}
Each cell reports class-macro prompt IoU with gain over the baseline in
parentheses. The full safe variant uses the bilinear residual router,
query+global+class features, an explicit no-op action, and a
held-out-calibrated no-op margin.}
\label{tab:app_ablation_results}
\resizebox{\linewidth}{!}{%
\begin{tabular}{@{}l r r r r r r r r r@{}}
\specialrule{0.1em}{0pt}{0pt}
\rowcolor{headerblue}
\textbf{Variant} & \textbf{ADE} & \textbf{CamVid} & \textbf{City.}
& \textbf{COCO} & \textbf{ImageNet} & \textbf{LoveDA}
& \textbf{PC-59} & \textbf{VOC} & \textbf{Avg.} \\
\specialrule{0em}{0pt}{0pt}
\midrule
\textbf{Full safe} & \textbf{0.443 (+8.7)} & \textbf{0.491 (+15.0)} & \textbf{0.611 (+4.8)} & \textbf{0.634 (+12.4)} & \textbf{0.881 (+11.1)} & \textbf{0.390 (+11.5)} & \textbf{0.699 (+9.0)} & \textbf{0.787 (+2.8)} & \textbf{0.617 (+9.4)} \\
\rowcolor{tableblue}
No calibration & 0.460 (+10.5) & 0.507 (+16.6) & 0.617 (+5.4) & 0.642 (+13.2) & 0.883 (+11.3) & 0.405 (+13.0) & 0.710 (+10.2) & 0.787 (+2.8) & 0.626 (+10.4) \\
No no-op & 0.450 (+9.4) & 0.491 (+15.0) & 0.616 (+5.3) & 0.640 (+13.0) & 0.881 (+11.1) & 0.409 (+13.4) & 0.710 (+10.2) & 0.787 (+2.8) & 0.623 (+10.0) \\
\rowcolor{tableblue}
MLP head & 0.422 (+6.6) & 0.467 (+12.6) & 0.587 (+2.4) & 0.618 (+10.7) & 0.881 (+11.1) & 0.369 (+9.3) & 0.702 (+9.3) & 0.785 (+2.6) & 0.604 (+8.1) \\
Linear head & 0.428 (+7.2) & 0.346 (+0.5) & 0.563 (+0.0) & 0.638 (+12.8) & 0.770 (+0.0) & 0.369 (+9.3) & 0.697 (+8.9) & 0.760 (+0.1) & 0.571 (+4.8) \\
\rowcolor{tableblue}
Query+global & 0.445 (+9.0) & 0.436 (+9.5) & 0.566 (+0.3) & 0.636 (+12.6) & 0.880 (+11.0) & 0.307 (+3.2) & 0.702 (+9.3) & 0.782 (+2.3) & 0.594 (+7.2) \\
Global only & 0.421 (+6.6) & 0.361 (+2.0) & 0.563 (+0.0) & 0.606 (+9.6) & 0.866 (+9.7) & 0.337 (+6.2) & 0.668 (+6.0) & 0.774 (+1.5) & 0.575 (+5.2) \\
\specialrule{0.1em}{0pt}{0pt}
\end{tabular}
}
\vspace{-0.4em}
\end{table}

\begin{table}[H]
\centering
\scriptsize
\setlength{\tabcolsep}{5.0pt}
\renewcommand{\arraystretch}{1.08}
\caption{\textbf{Macro-averaged safety and recovery diagnostics.}
No calibration keeps the no-op action but sets the margin to zero. No
no-op removes the baseline action entirely, so the router must select a
residual query whenever one is available. Tail $\Delta$ is the mean
per-prompt gain over the worst quartile of held-out prompts.}
\label{tab:app_sam3_ablation_safety}
\resizebox{\linewidth}{!}{%
\begin{tabular}{@{}l r r r r r r r r@{}}
\specialrule{0.1em}{0pt}{0pt}
\rowcolor{headerblue}
\textbf{Variant} & \textbf{Avg. IoU} & \textbf{$\Delta$ pp}
& \textbf{Tail $\Delta$} & \textbf{Gap pp} & \textbf{Rec.} & \textbf{Keep def.}
& \textbf{Action} & \textbf{Harmful} \\
\midrule
\textbf{Full safe} & \textbf{0.617} & \textbf{+9.4} & \textbf{$-0.1$} & +19.8 & \textbf{49.4\%} & \textbf{77.2\%} & 22.8\% & \textbf{0.30\%} \\
\rowcolor{tableblue}
No calibration & 0.626 & +10.4 & $-2.4$ & +19.8 & 53.6\% & 15.0\% & 85.0\% & 15.50\% \\
No no-op & 0.623 & +10.0 & $-2.8$ & +19.8 & 52.2\% & 0.0\% & 100.0\% & 21.33\% \\
\rowcolor{tableblue}
MLP head & 0.604 & +8.1 & +0.0 & +19.8 & 43.0\% & 82.6\% & 17.4\% & 0.07\% \\
Linear head & 0.571 & +4.8 & +0.0 & +19.8 & 20.7\% & 84.7\% & 15.3\% & 0.21\% \\
\rowcolor{tableblue}
Query+global & 0.594 & +7.2 & $-0.1$ & +19.8 & 38.5\% & 83.5\% & 16.5\% & 0.49\% \\
Global only & 0.575 & +5.2 & $-0.1$ & +19.8 & 28.0\% & 84.6\% & 15.4\% & 0.16\% \\
\specialrule{0.1em}{0pt}{0pt}
\end{tabular}
}
\vspace{-0.8em}
\end{table}

Tables~\ref{tab:app_ablation_results}
and~\ref{tab:app_sam3_ablation_safety} are negative controls: removing
calibration or the no-op action can increase the mean, but it raises
harmful-action rates and damages lower-tail prompts. The main method
therefore keeps the baseline as an explicit safe action and calibrates
the margin on held-out images.

\begin{figure}[t]
\centering
\includegraphics[width=0.58\linewidth]{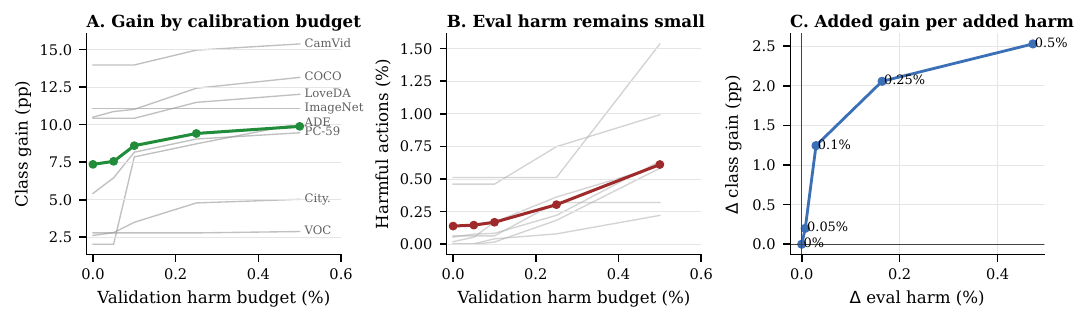}\hfill
\includegraphics[width=0.39\linewidth]{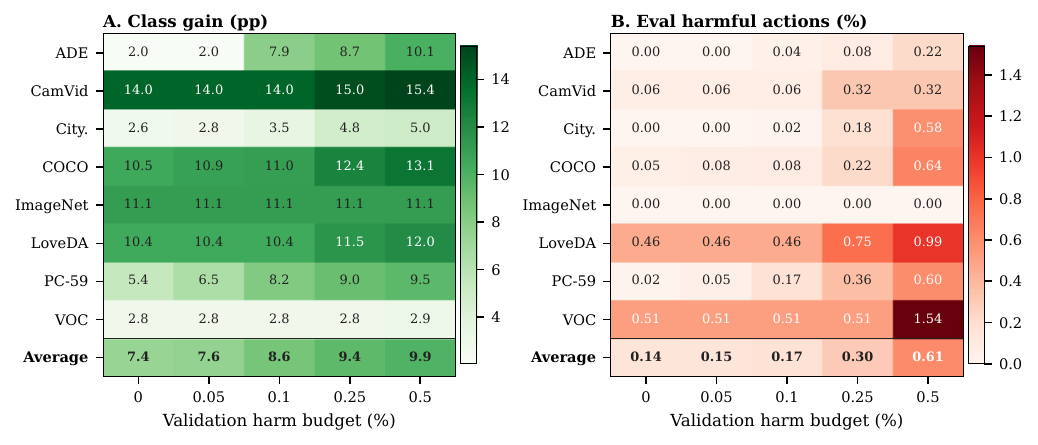}
\vspace{-0.6em}
\caption{\textbf{SAM~3 calibration-budget ablation.}
Left: sweeping the validation harmful-action budget trades a small
increase in label risk for additional class-macro gain; the main paper
uses the $0.25\%$ setting. Right: per-dataset class-gain heatmap with
the dataset average row.}
\label{fig:app_calibration_budget_ablation}
\end{figure}

\clearpage
\section{Query-Routing Diagnostics}
\label{app:plots}

Each selector run writes the same diagnostic suite, but the appendix only
reproduces representative plots that directly support the main claims:
query-use concentration, qualitative recoveries, class-level gain
structure, transfer controls, score--quality calibration, and online
overhead. The complete generated artifacts remain in
\texttt{results\_tables/} for reproducibility.

\subsection{SAM~3 Query-Use Concentration}
\label{app:sam3_query_use}

\begin{figure}[H]
    \centering
    \includegraphics[width=\linewidth]{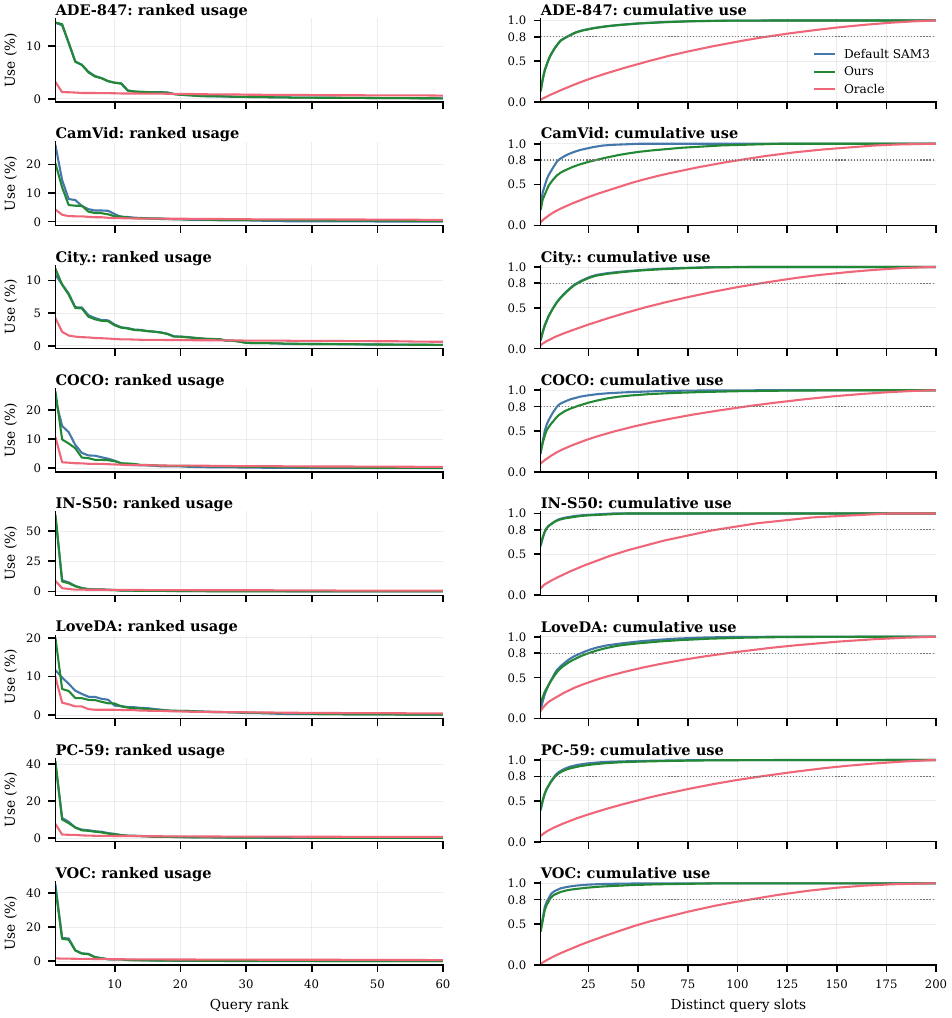}
    \caption{\textbf{SAM~3 query-use distributions for corrected no-op runs.}
    Left column: ranked query-use distributions show the concentration of
    the default decoder, the final \ours{} output after no-op calibration,
    and the oracle. Right column: cumulative curves show that \ours{} uses
    a broader final-query pool than the default selection while remaining more
    conservative than the oracle.}
    \label{fig:app_sam3_query_usage}
\end{figure}

The ranked and cumulative query-use views in
Figure~\ref{fig:app_sam3_query_usage} provide the fuller evidence
referenced in the main text. The raw diagnostic folders also contain
query-index histograms, agreement heatmaps, radar plots, and per-class
views; these two views are the most compact because they show the same
fact without relying on a particular query index. The default selection
repeatedly selects a small subset of slots, whereas the oracle and the
learned final output draw from a broader pool after the no-op gate.

\FloatBarrier
\subsection{Qualitative Recoveries}
\label{app:qualitative_recoveries}

\begin{figure}[H]
    \centering
    \includegraphics[width=0.94\linewidth]{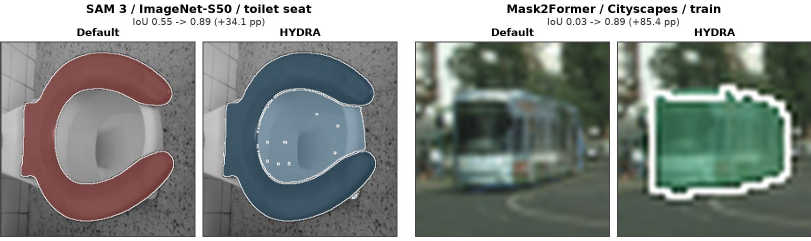}
    \caption{\textbf{Nonempty defaults can still improve.}
    Each pair shows only the released default and the routed \ours{} output,
    cropped around the corrected region. The default already overlaps the
    target in both examples, but the frozen candidate selected by \ours{}
    better completes the intended region.}
    \label{fig:app_qualitative_improvement_row}
\end{figure}

The examples in Figure~\ref{fig:app_qualitative_improvement_row} show
the stricter before/after case for both SAM~3 and Mask2Former: the
default output is already nonempty, but the routed candidate better
matches the intended region. These examples are representative cases,
not the basis for the aggregate claim.

\begin{figure}[H]
    \centering
    \includegraphics[width=0.92\linewidth]{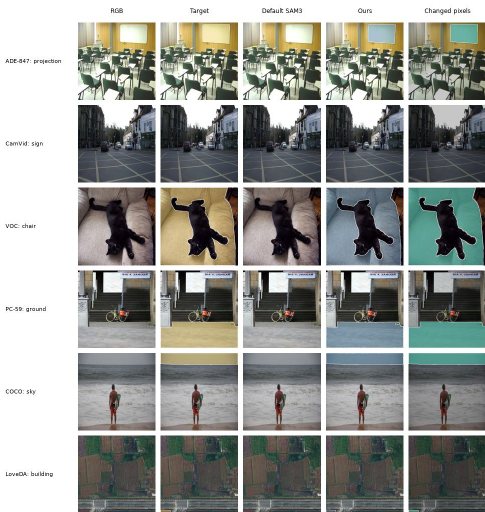}
    \caption{\textbf{Additional corrected SAM~3 examples.}
    Best-case qualitative examples from the audited no-op reports. Each row
    shows RGB, target mask, default SAM~3, our routed output, and changed
    pixels; the examples are selected for visual clarity from high-gain
    held-out prompts.}
    \label{fig:app_sam3_best_examples}
\end{figure}

\begin{figure}[H]
    \centering
    \includegraphics[width=\linewidth]{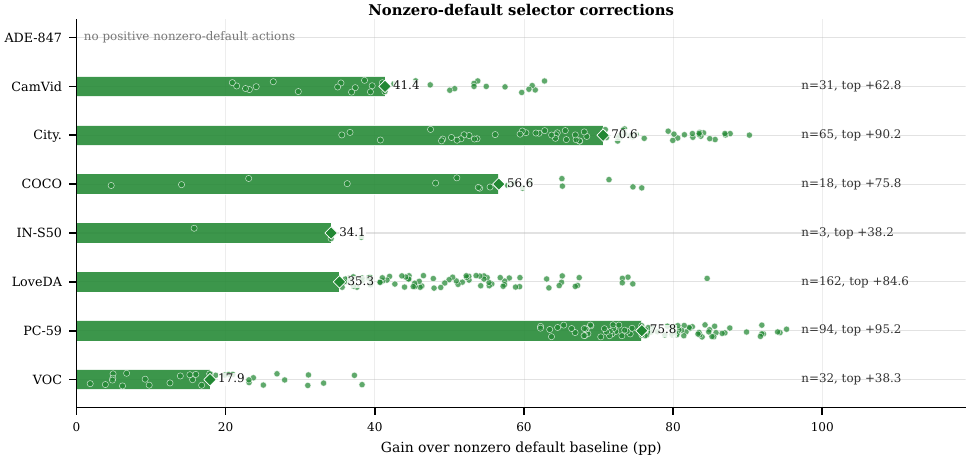}
    \caption{\textbf{SAM~3 corrections when the default output is nonzero.}
    Each row shows held-out prompts where default SAM~3 already produced a
    nonempty mask and the calibrated no-op selector still improved prompt IoU.
    Points are prompt-level gains; the green bar and diamond mark the median
    positive nonzero-default gain.}
    \label{fig:app_sam3_nonzero_baseline_gains}
\end{figure}

Figure~\ref{fig:app_sam3_best_examples} gives additional high-gain
SAM~3 examples from ADE-847, CamVid, VOC, PC-59, COCO, and LoveDA.
Figure~\ref{fig:app_sam3_nonzero_baseline_gains} quantifies the same
nonempty-default case for SAM~3, confirming that the gains are not only
from turning empty defaults into nonempty masks.

\FloatBarrier
\subsection{Class-Level Gain Structure}
\label{app:sam3_class_level_gains}

\begin{figure}[H]
    \centering
    \includegraphics[width=\linewidth]{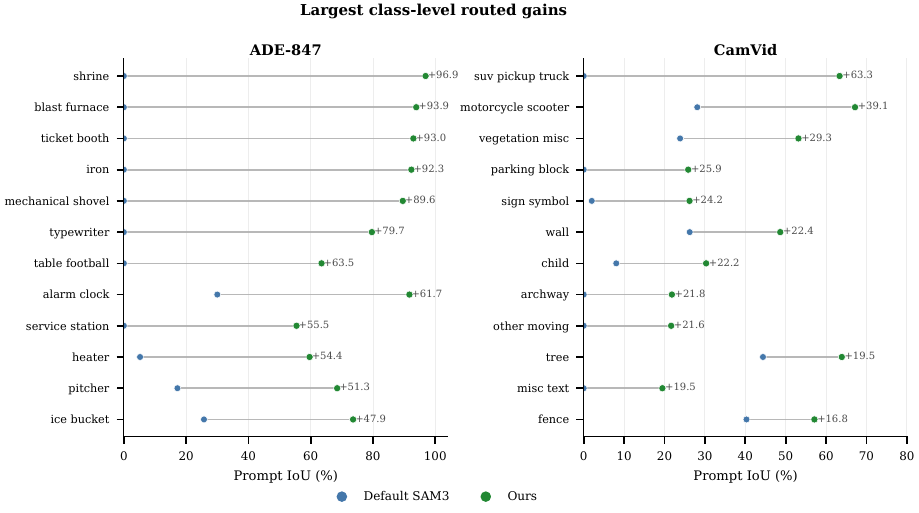}
    \caption{\textbf{SAM~3 class-level gains across promptable datasets (1/4).}
    ADE-847 and CamVid classes with the largest routed gains. Each row
    compares the default SAM~3 prompt IoU to the final no-op-routed output.}
    \label{fig:app_sam3_per_class_gains_1}
\end{figure}

\begin{figure}[H]
    \centering
    \includegraphics[width=\linewidth]{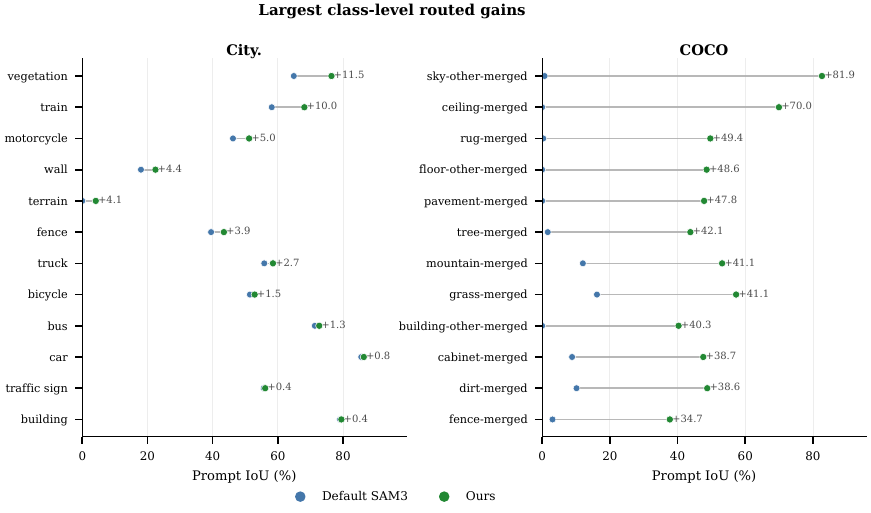}
    \caption{\textbf{SAM~3 class-level gains across promptable datasets (2/4).}
    Cityscapes and COCO classes with the largest routed gains.}
    \label{fig:app_sam3_per_class_gains_2}
\end{figure}

\begin{figure}[H]
    \centering
    \includegraphics[width=\linewidth]{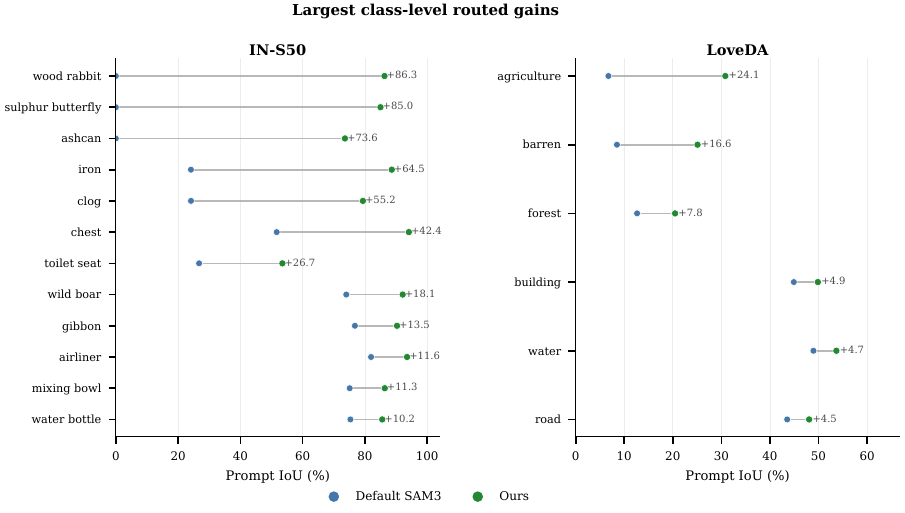}
    \caption{\textbf{SAM~3 class-level gains across promptable datasets (3/4).}
    ImageNet-S50 and LoveDA classes with the largest routed gains.}
    \label{fig:app_sam3_per_class_gains_3}
\end{figure}

\begin{figure}[H]
    \centering
    \includegraphics[width=\linewidth]{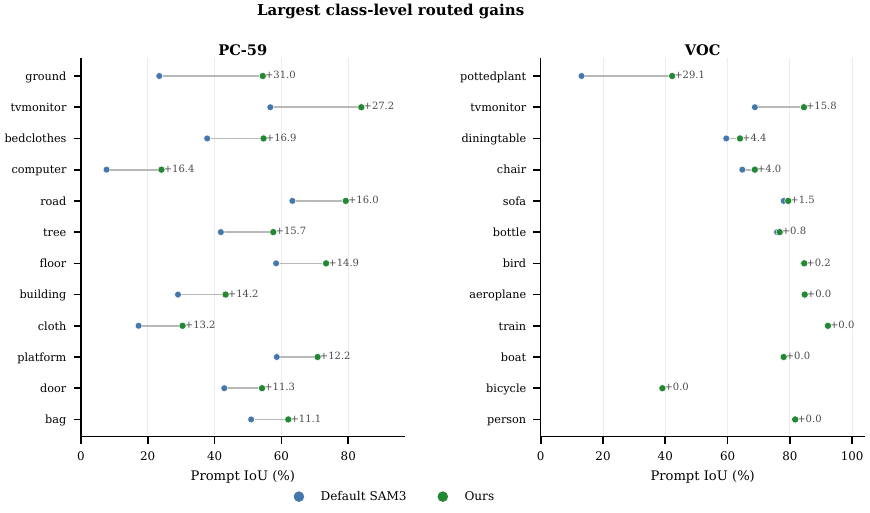}
    \caption{\textbf{SAM~3 class-level gains across promptable datasets (4/4).}
    PC-59 and Pascal VOC classes with the largest routed gains. Full
    per-class CSVs for every dataset are included in the evidence bundle.}
    \label{fig:app_sam3_pc59_class_gains}
\end{figure}

Figures~\ref{fig:app_sam3_per_class_gains_1}--\ref{fig:app_sam3_pc59_class_gains}
show where the SAM~3 gains are allocated across prompt labels. The plots
are class-level diagnostics rather than aggregate performance claims:
the tables in Appendix~\ref{app:full_results} give the averages, while
these figures show that the improvements are spread across many labels
and are not explained by a single easy class.

\FloatBarrier
\subsection{Cross-Domain Transfer Controls}
\label{app:transfer_controls}

\begin{figure}[H]
\centering
\includegraphics[width=0.98\linewidth]{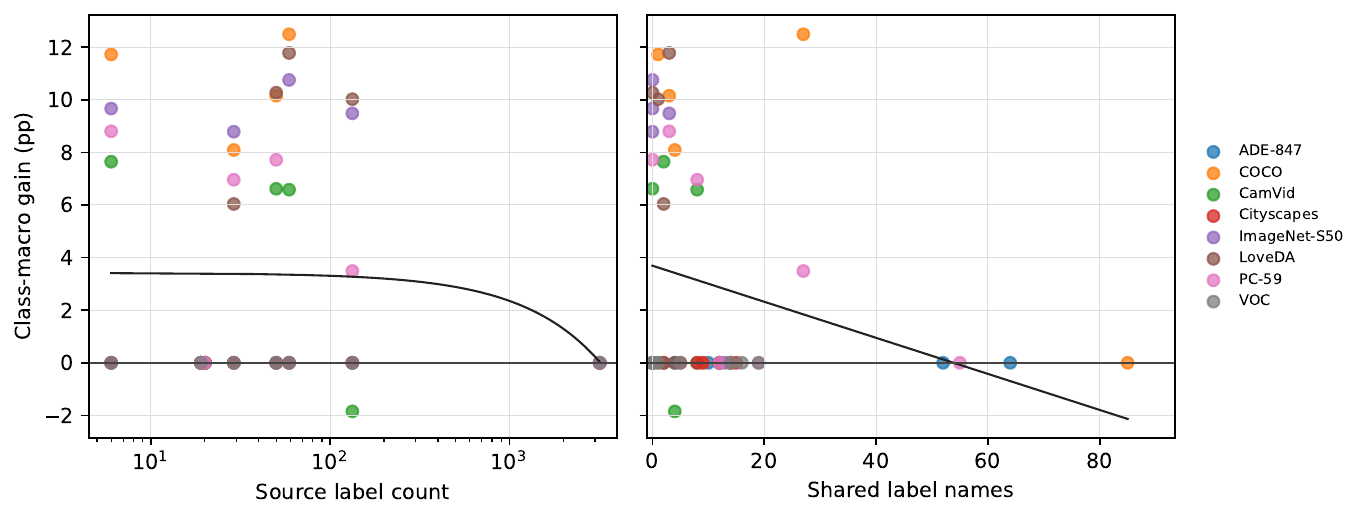}
\caption{\textbf{Label-count controls for SAM~3 transfer.}
Off-diagonal cross-domain selector runs plotted against source vocabulary
size and target-overlap count. Transfer gain is not explained by either
more source labels or more shared label names; calibrated no-op runs
appear at zero.}
\label{fig:app_sam3_label_count_vs_gain}
\end{figure}

Figure~\ref{fig:app_sam3_label_count_vs_gain} checks whether the
cross-domain heatmap is explained by vocabulary size or name overlap. It
is not: across off-diagonal runs, source-label count and shared-name
count correlate weakly negatively with class-macro gain ($r=-0.25$ and
$r=-0.26$). Compact or medium-vocabulary sources such as LoveDA, CamVid,
ImageNet-S50, and PC-59 transfer well, while ADE-847 often falls on the
calibrated no-op band. Several strong pairs share few or no exact names,
and high-overlap pairs can stay at zero when no-op is safer. The transfer
therefore reflects reusable mask and feature geometry plus a conservative
gate, not lexical memorization.

\FloatBarrier
\subsection{\texorpdfstring{COCO$\to$CamVid Negative Transfer}{COCO to CamVid Negative Transfer}}
\label{app:coco_camvid_negative_transfer}

The only negative off-diagonal transfer entry occurs when the selector
trained on COCO is applied to CamVid. We isolate this case to distinguish
limited target-domain headroom from a domain-shifted release decision.
The COCO-trained selector is evaluated on CamVid without retraining,
renormalization, vocabulary remapping, or margin recalibration. This is a
strict transfer test: if the source gate releases a residual action, it
must be correct under the target-domain score geometry.

\begin{figure}[H]
    \centering
    \includegraphics[width=\linewidth]{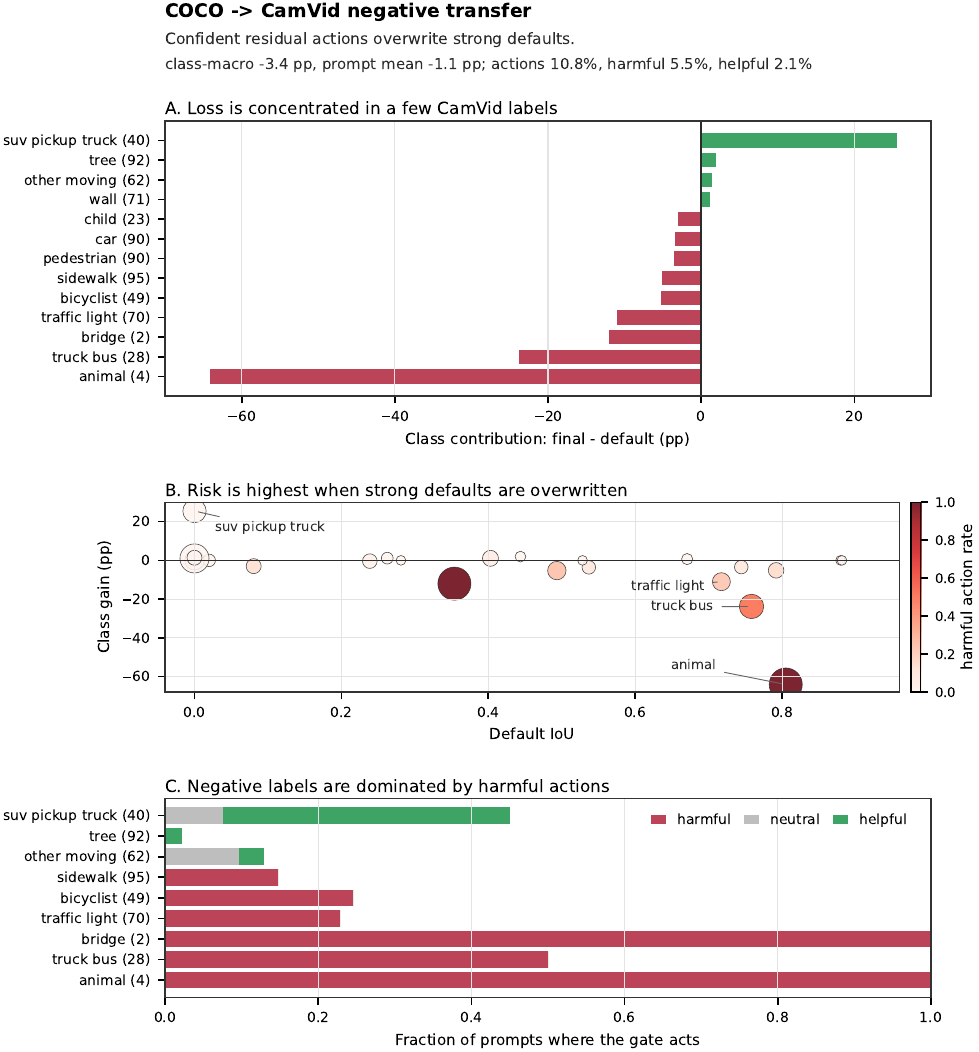}
    \caption{\textbf{Why COCO$\to$CamVid is the negative transfer cell.}
    The source selector keeps the COCO checkpoint, normalization, vocabulary,
    and no-op margin, then applies them directly to CamVid. Losses are not
    diffuse: a few CamVid road-scene labels receive confident residual actions
    that overwrite already strong defaults. The cell is therefore a calibration
    failure under target-domain shift, not evidence that the frozen SAM~3 pool
    lacks useful CamVid masks.}
    \label{fig:app_coco_camvid_negative_transfer}
\end{figure}

\begin{figure}[H]
    \centering
    \includegraphics[width=0.74\linewidth]{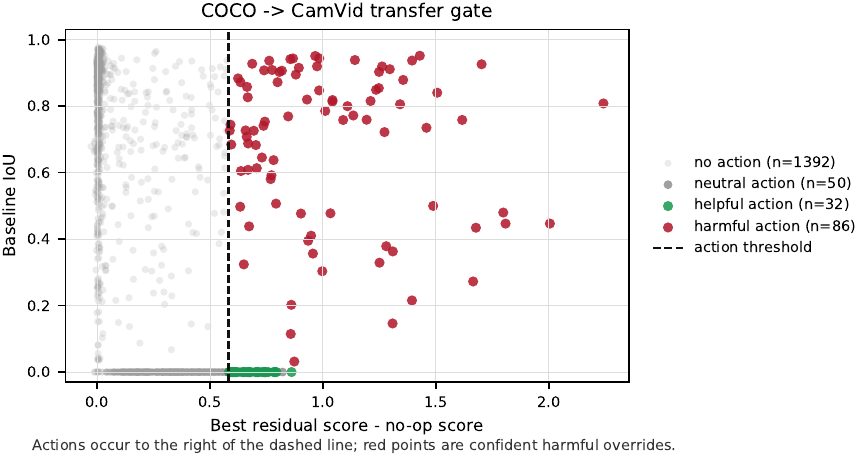}
    \caption{\textbf{The failed actions are confident overrides.}
    Actions occur to the right of the dashed margin. In COCO$\to$CamVid, many
    acted-upon prompts with high default IoU are harmful, so the learned
    residual score is miscalibrated relative to the no-op action on this target.
    This explains why the conservative gate preserves many prompts but still
    loses on the small subset it releases.}
    \label{fig:app_coco_camvid_negative_transfer_gate}
\end{figure}

Figures~\ref{fig:app_coco_camvid_negative_transfer}
and~\ref{fig:app_coco_camvid_negative_transfer_gate} show that the cell
does not fail because CamVid lacks oracle headroom. It fails because the
source-trained residual gate sometimes treats a high-scoring residual as
safer than a strong CamVid default. The error is concentrated in a few
road-scene labels, and the acted-upon harmful cases are confident rather
than borderline. The right interpretation is therefore miscalibrated
action release under domain shift, not a collapse of the frozen SAM~3
mask pool.

\FloatBarrier
\subsection{Score--Quality and Routing Structure}
\label{app:score_quality_routing}

\begin{figure}[t]
    \centering
    \includegraphics[width=0.98\linewidth]{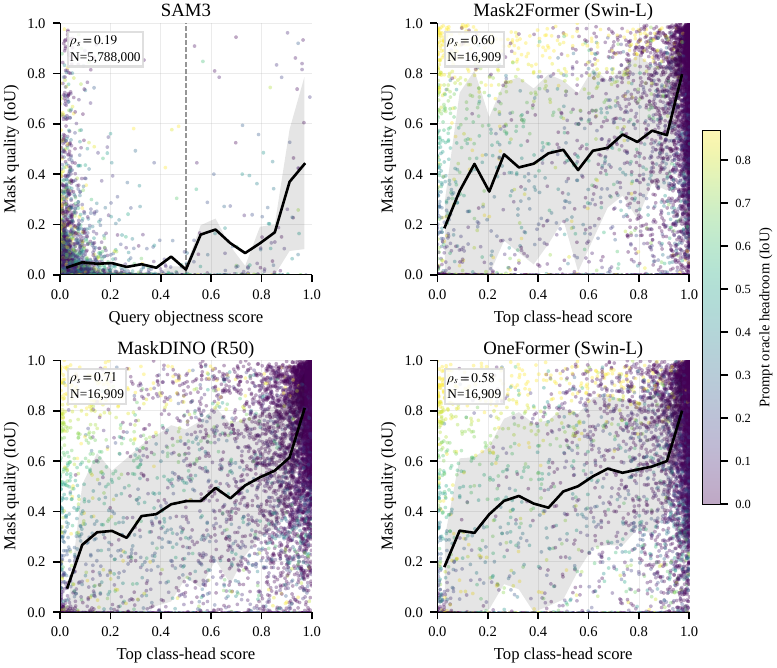}
    \vspace{-0.7em}
    \caption{\textbf{Full score--quality diagnostic.}
    This expands Figure~\ref{fig:score_quality_gap_main} to all evaluated
    architectures. Each point is an already-computed candidate mask: a
    SAM~3 query for ADE-847, or the top class-head query for a DETR-family
    segmenter on ADE20k. The x-axis is the model's own decision score,
    the y-axis is candidate mask IoU, color denotes prompt-level oracle
    headroom, and the black curve shows binned mean quality with shaded
    interquartile range.}
    \label{fig:score_quality_gap_appendix}
    \vspace{-0.8em}
\end{figure}

\begin{figure}[t]
    \centering
    \includegraphics[width=0.94\linewidth]{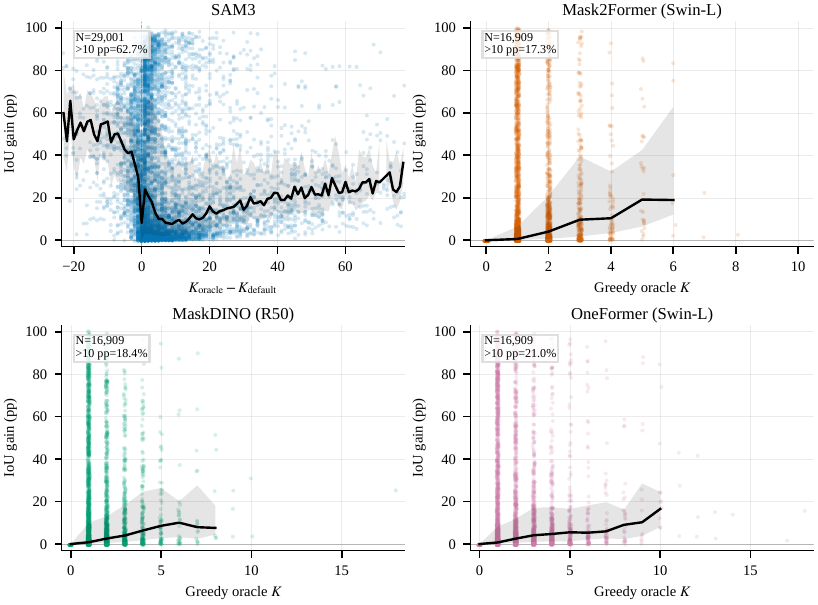}
    \vspace{-0.7em}
    \caption{\textbf{Oracle gains do not require a wholesale decoder
    replacement.}
    Each point is one evaluated prompt. The x-axis is the number of
    oracle-selected queries for class-aware models, and the change in
    selected query count relative to the default thresholded mask set for
    SAM~3. Negative SAM~3 values mean that the baseline threshold
    selected more masks than the oracle needed, so the oracle is
    effectively replacing or pruning over-selected default candidates.
    The y-axis is the per-prompt IoU gain of the greedy oracle.
    Black curves show the binned median with interquartile range. Large
    gains frequently occur at small query counts or small count changes,
    supporting the view that the failure is a local selection/routing
    problem over existing candidates.}
    \label{fig:selection_vs_improvement_appendix}
    \vspace{-0.8em}
\end{figure}

\begin{figure}[p]
    \centering
    \includegraphics[width=0.98\linewidth]{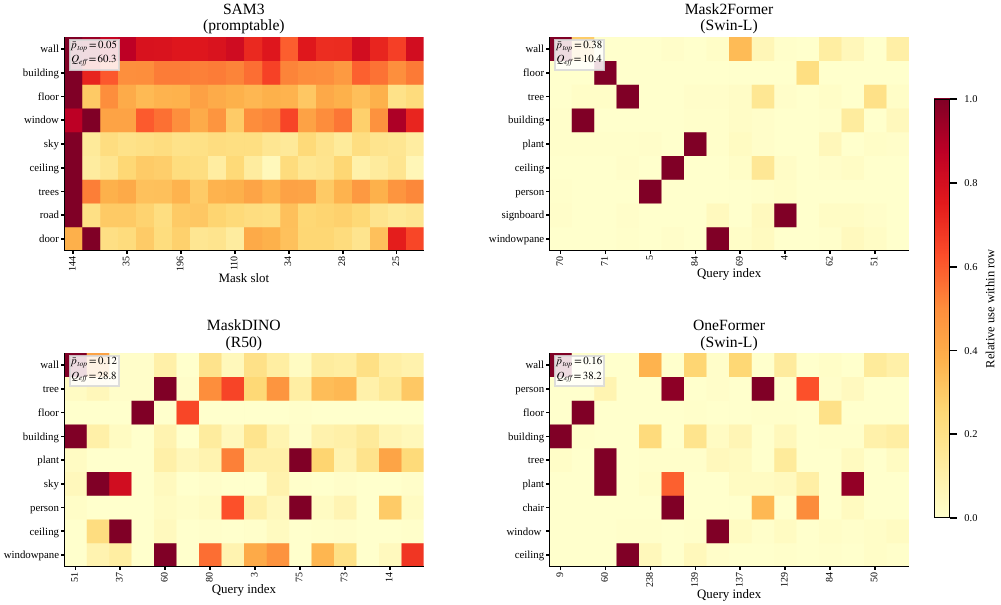}
    \vspace{-0.7em}
    \caption{\textbf{Prompt and class labels induce different query-use
    regimes.}
    Rows are frequent labels, columns are the most used query or mask-slot
    indices for each model, and colors are normalized within each row.
    SAM~3 is promptable, so its panel shows ADE-847 text prompts against
    prompt-conditioned mask slots; the broad activation across many slots
    reflects more diverse candidate use. The class-aware DETR-family models
    instead show class--query coupling, especially Mask2Former, where a
    semantic class often repeatedly relies on one slot or a small set of
    slots.}
    \label{fig:class_query_specialization_appendix}
    \vspace{-0.8em}
\end{figure}

The score--quality panels in Figure~\ref{fig:score_quality_gap_appendix}
extend the main diagnostic across SAM~3 and the universal segmenters.
Figures~\ref{fig:selection_vs_improvement_appendix}
and~\ref{fig:class_query_specialization_appendix} then make the routing
structure explicit. The greedy oracle often obtains large gains with a
small number of query changes, and the label-to-query map differs by
training regime. SAM~3 spreads ADE-847 text prompts over many
prompt-conditioned mask slots, while the class-aware DETR-family models
repeatedly use particular query slots for particular semantic classes.
This does not mean every query is permanently tied to one label, but it
does show that the query pool is structured: a post-hoc router can
exploit label- and query-dependent regularities that are invisible to a
single scalar score.

\FloatBarrier
\subsection{Online Overhead Details}
\label{app:overhead_details}

\begin{figure}[t]
\centering
\resizebox{0.98\linewidth}{!}{\definecolor{hydraBlue}{HTML}{4477AA}
\definecolor{hydraGreen}{HTML}{228833}
\definecolor{hydraText}{HTML}{222222}
\definecolor{hydraGrid}{HTML}{D8D8D8}

\begin{tikzpicture}
\pgfplotsset{
  hydra axis/.style={
    axis x line*=bottom,
    axis y line*=left,
    axis line style={draw=black, line width=0.65pt},
    tick style={draw=black, line width=0.65pt},
    major tick length=2.2pt,
    tick label style={font=\small, text=black},
    label style={font=\normalsize, text=black},
    title style={font=\bfseries\large, text=black, yshift=4pt},
    grid=major,
    grid style={draw=hydraGrid, line width=0.35pt},
    ymajorgrids=true,
    xmajorgrids=false,
    clip=false,
  }
}

\begin{groupplot}[
  group style={
    group size=2 by 1,
    horizontal sep=1.15cm
  },
]

\nextgroupplot[
  hydra axis,
  width=7.60cm,
  height=4.55cm,
  title={(a) End-to-end live latency},
  ylabel={Latency / input (ms)},
  ymin=0,
  ymax=285,
  xmin=-0.55,
  xmax=3.55,
  ytick={0,50,100,150,200,250},
  xtick={0,1,2,3},
  xticklabels={SAM3,M2F,OneF.,MaskDINO},
  xticklabel style={font=\small, align=center, text=black},
  enlarge x limits=false,
  legend style={
    draw=none,
    fill=none,
    at={(0.98,0.98)},
    anchor=north east,
    legend columns=1,
    font=\small,
    /tikz/every even column/.append style={column sep=0.35cm}
  },
]

\addlegendimage{area legend, fill=hydraBlue, draw=black, line width=0.35pt}
\addlegendentry{baseline}
\addlegendimage{area legend, fill=hydraGreen, draw=black, line width=0.35pt}
\addlegendentry{+ HYDRA}

\addplot+[
  ybar,
  mark=none,
  bar width=5.8pt,
  bar shift=-3.8pt,
  fill=hydraBlue,
  draw=black,
  line width=0.35pt,
] coordinates {
  (0,214.466)
  (1,86.117)
  (2,103.224)
  (3,43.689)
};

\addplot+[
  ybar,
  mark=none,
  bar width=5.8pt,
  bar shift=3.8pt,
  fill=hydraGreen,
  draw=black,
  line width=0.35pt,
] coordinates {
  (0,242.582)
  (1,94.424)
  (2,112.757)
  (3,52.039)
};

\node[align=center, font=\scriptsize, text=black]
  at (axis cs:0,263) {+28.1 ms\\+13.1\%};
\node[align=center, font=\scriptsize, text=black]
  at (axis cs:1,114) {+8.3 ms\\+9.6\%};
\node[align=center, font=\scriptsize, text=black]
  at (axis cs:2,132) {+9.5 ms\\+9.2\%};
\node[align=center, font=\scriptsize, text=black]
  at (axis cs:3,72) {+8.4 ms\\+19.1\%};

\nextgroupplot[
  hydra axis,
  width=5.85cm,
  height=4.55cm,
  title={(b) Online throughput},
  ylabel={Inputs / s},
  ymin=0,
  ymax=27.5,
  ytick={0,5,10,15,20,25},
  xtick={0,1,2,3},
  xticklabels={SAM3,M2F,OneF.,MaskDINO},
  xticklabel style={font=\small, align=center, text=black},
  xmin=-0.55,
  xmax=3.55,
  enlarge x limits=false,
]

\addplot+[
  ybar,
  mark=none,
  bar width=5.8pt,
  bar shift=-3.8pt,
  fill=hydraBlue,
  draw=black,
  line width=0.35pt,
] coordinates {
  (0,4.663)
  (1,11.612)
  (2,9.688)
  (3,22.889)
};

\addplot+[
  ybar,
  mark=none,
  bar width=5.8pt,
  bar shift=3.8pt,
  fill=hydraGreen,
  draw=black,
  line width=0.35pt,
] coordinates {
  (0,4.122)
  (1,10.591)
  (2,8.869)
  (3,19.216)
};

\node[align=center, font=\scriptsize, text=black]
  at (axis cs:0,6.0) {-11.6\%};
\node[align=center, font=\scriptsize, text=black]
  at (axis cs:1,13.0) {-8.8\%};
\node[align=center, font=\scriptsize, text=black]
  at (axis cs:2,11.0) {-8.5\%};
\node[align=center, font=\scriptsize, text=black]
  at (axis cs:3,25.0) {-16.0\%};

\end{groupplot}
\end{tikzpicture}}
\caption{\textbf{Detailed online overhead.}
We time the same 100 live inputs summarized in
Figure~\ref{fig:selector_overhead_main_wrap}.
The left panel reports end-to-end latency for the native baseline and the
selector path; the right panel reports absolute online throughput. SAM~3 uses
ImageNet-S50 prompts, and the universal segmenters use ADE20k validation
images.}
\label{fig:app_selector_overhead_details}
\end{figure}
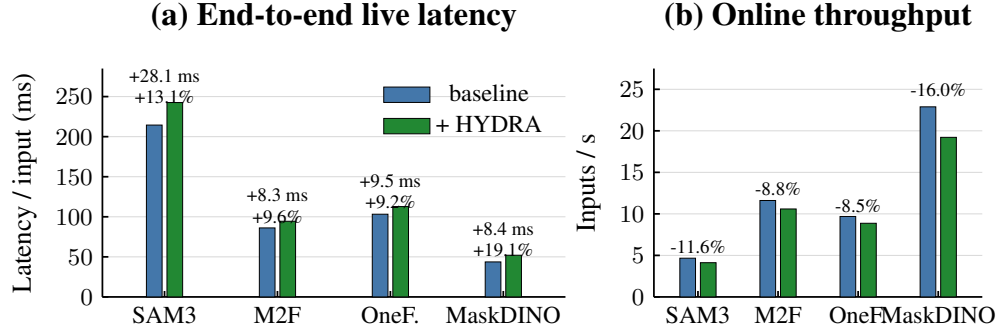

The detailed latency and throughput measurements in
Figure~\ref{fig:app_selector_overhead_details} support the compact
online-cost summary in the main paper. The selector adds a small online
decision on top of cached frozen outputs; the dominant cost remains the
underlying segmenter forward pass.

\FloatBarrier
\section{Controlled Mechanism Study}
\label{app:controlled}

\subsection{Controlled Setup}
\label{app:controlled_setup}

Images are $64\times64$ binary occupancy maps drawn from
$K\in\{2,3,4\}$ shape classes: horizontal ellipses, vertical ellipses,
rings, and crosses. TinyDETR uses a 3-layer CNN encoder, learned object
queries, cross-attention, a per-query mask MLP, and a per-query score
head. The model is trained with Hungarian matching and BCE+Dice+score
loss.

Distribution shift is parameterized by an ambiguity variable
$\rho\in[0,1]$ through a superellipse:
\begin{equation}
    \mathcal{S}_\rho
    =
    \Bigl\{R_\theta(u,v)^\top:
    |u/a|^{p(\rho)}+|v/b|^{p(\rho)}\leq 1\Bigr\},
    \quad
    p(\rho)=\frac{2}{1-\rho+\epsilon}.
    \label{eq:superellipse}
\end{equation}
The endpoints correspond to training regimes; intermediate values are
out-of-distribution mixtures that activate competing specialized queries.
We use this setting as a microscope rather than a benchmark. The goal is
to test the three pieces of the paper's mechanism in isolation: matching
must create query specialization; the score decoder must fail most when
two specialized queries are both plausible; and a router must recover the
lost quality from frozen query features rather than from extra image
features. Sweeping $K$ is the stress test. With $K{=}2$ there is no spare
query, with $K{=}3$ the boundary is asymmetric because one class is not
part of the horizontal--vertical interpolation, and with $K{=}4$ the
model has additional specialized queries that could absorb or confuse
the boundary. A mechanism that appears in all three regimes is therefore
not a byproduct of one convenient toy cardinality.

\subsection{Scaling Across Query-Set Sizes}
\label{app:toy_scaling}

\begin{figure}[H]
    \centering
    \includegraphics[width=0.96\linewidth]{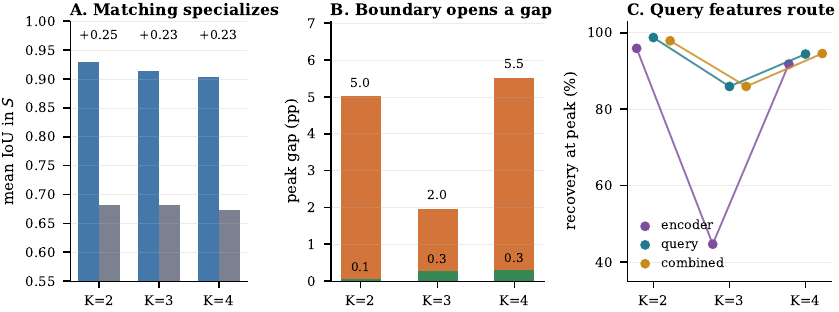}
    \vspace{-0.6em}
    \caption{%
        \textbf{Controlled mechanism sweep over $K\in\{2,3,4\}$.}
        \emph{A}: Hungarian matching produces higher diagonal than
        off-diagonal entries in the specialization matrix $S$ for every
        $K$. \emph{B}: the score-argmax decoder opens a boundary gap
        (orange), while the query-feature router leaves only a small
        residual gap (green). \emph{C}: at the hardest boundary point,
        routers using query features recover more of the oracle gap than
        encoder-only routing, especially for $K{=}3$.%
    }
    \label{fig:toy_scaling_k234}
    \vspace{-0.8em}
\end{figure}

Figure~\ref{fig:toy_scaling_k234} verifies that the mechanism is not an
artifact of choosing four queries. Matching produces a stable
specialization gap for $K{=}2,3,4$ (diagonal mean $0.93,0.91,0.90$
versus off-diagonal mean $0.68,0.68,0.67$). The score decoder opens a
boundary oracle gap in each case: $+5.0$ pp for $K{=}2$, $+2.0$ pp for
$K{=}3$, and $+5.5$ pp for $K{=}4$. At the same peak boundary points,
the query-feature router leaves only $+0.1,+0.3,+0.3$ pp residual gap.
The lower $K{=}3$ gap is useful rather than problematic: it shows that
the mechanism persists even when the score head is closer to correct.
The oracle itself remains high at those boundary points
($0.927,0.936,0.925$ IoU for $K{=}2,3,4$), so the failure is not a
collapse of mask generation. It is a selection error: the score-selected
query is coherent but not the best available query. This is the same
distinction made in the main empirical tables, where \ours{} improves a
frozen cache without asking the backbone to synthesize new masks.

\FloatBarrier
\subsection{Specialization During Training}
\label{app:specialization_emergence}

\begin{figure}[H]
    \centering
    \begin{subfigure}[t]{0.90\linewidth}
        \centering
        \includegraphics[width=\linewidth]{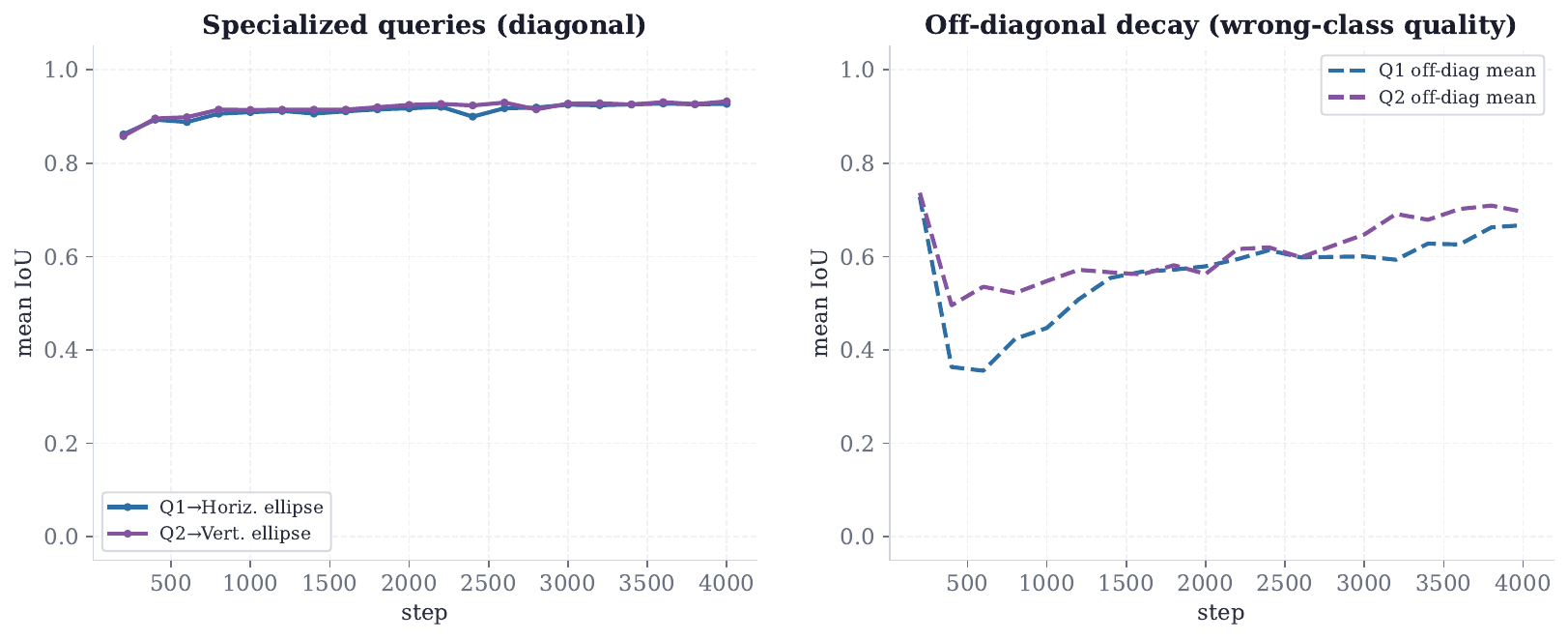}
        \caption{$K{=}2$}
        \label{fig:emergence_k2}
    \end{subfigure}
    \vspace{-0.35em}
    \begin{subfigure}[t]{0.90\linewidth}
        \centering
        \includegraphics[width=\linewidth]{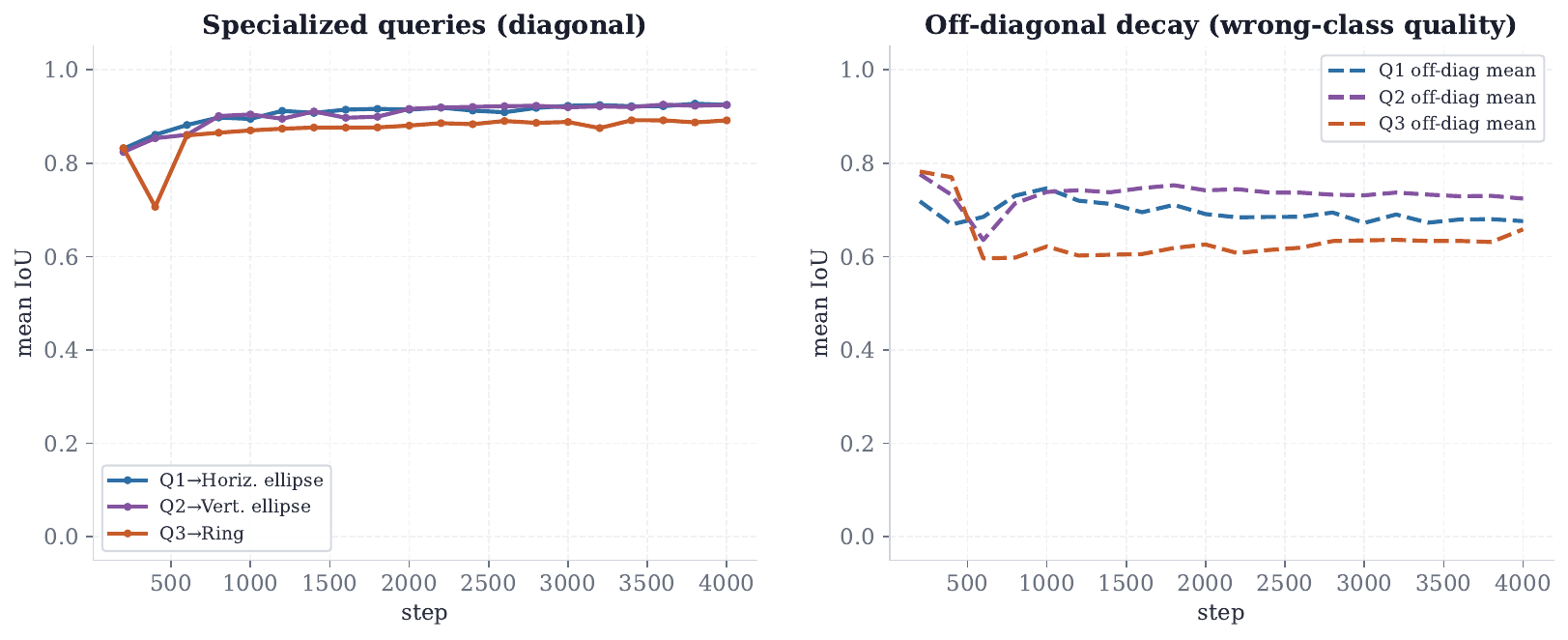}
        \caption{$K{=}3$}
        \label{fig:emergence_k3}
    \end{subfigure}
    \vspace{-0.35em}
    \begin{subfigure}[t]{0.90\linewidth}
        \centering
        \includegraphics[width=\linewidth]{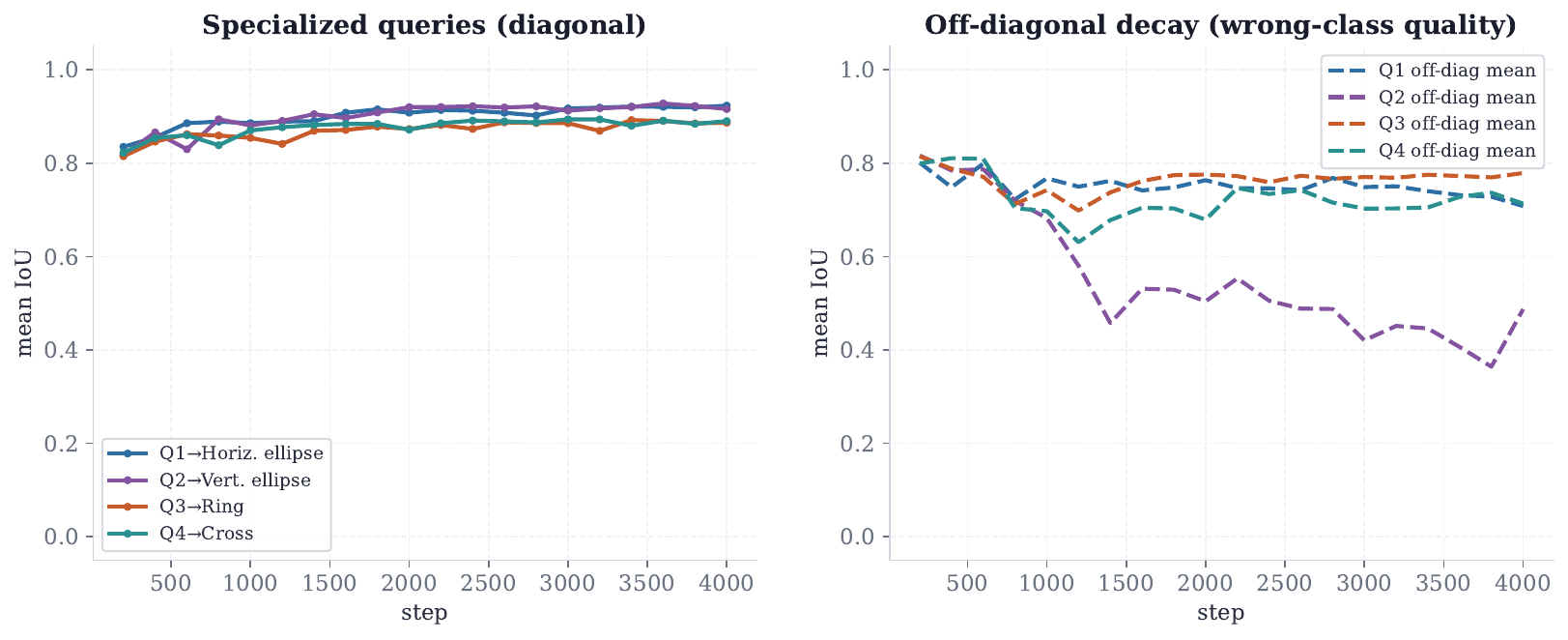}
        \caption{$K{=}4$}
        \label{fig:emergence_k4}
    \end{subfigure}
    \vspace{-0.5em}
    \caption{%
        \textbf{Specialization dynamics across query-set sizes.}
        In every controlled setting, diagonal entries of $\mathbf{S}$
        rise as queries commit to classes, while off-diagonal entries
        decay. The same matching feedback loop appears for
        $K{=}2,3,4$.%
    }
    \label{fig:emergence}
    \vspace{-0.8em}
\end{figure}

Figure~\ref{fig:emergence} expands the training trajectory for
$K{=}2,3,4$: diagonal entries of the specialization matrix rise as each
query commits to a class, while off-diagonal entries decay. This is the
concrete realization of the gradient-selection argument in
Section~\ref{sec:inevitability}; the effect is not specific to the
four-class visualization used in the main text.
The important point is temporal. The queries separate during ordinary
matched training, before any post-hoc router is introduced. Once that
separation exists, a scalar score has to compare masks produced by
different specialized query states. The score objective supervises
whether a query has found an object, but it does not supervise the
counterfactual question the oracle asks: among the candidate masks already
computed for this input, which one has the highest IoU? The controlled
trajectory therefore links the training objective to the later routing
problem instead of treating the oracle gap as an unexplained empirical
curiosity.

\FloatBarrier
\subsection{Boundary Concentration}
\label{app:boundary_concentration}

\begin{figure}[H]
    \centering
    \begin{subfigure}[t]{0.48\linewidth}
        \centering
        \includegraphics[width=\linewidth]{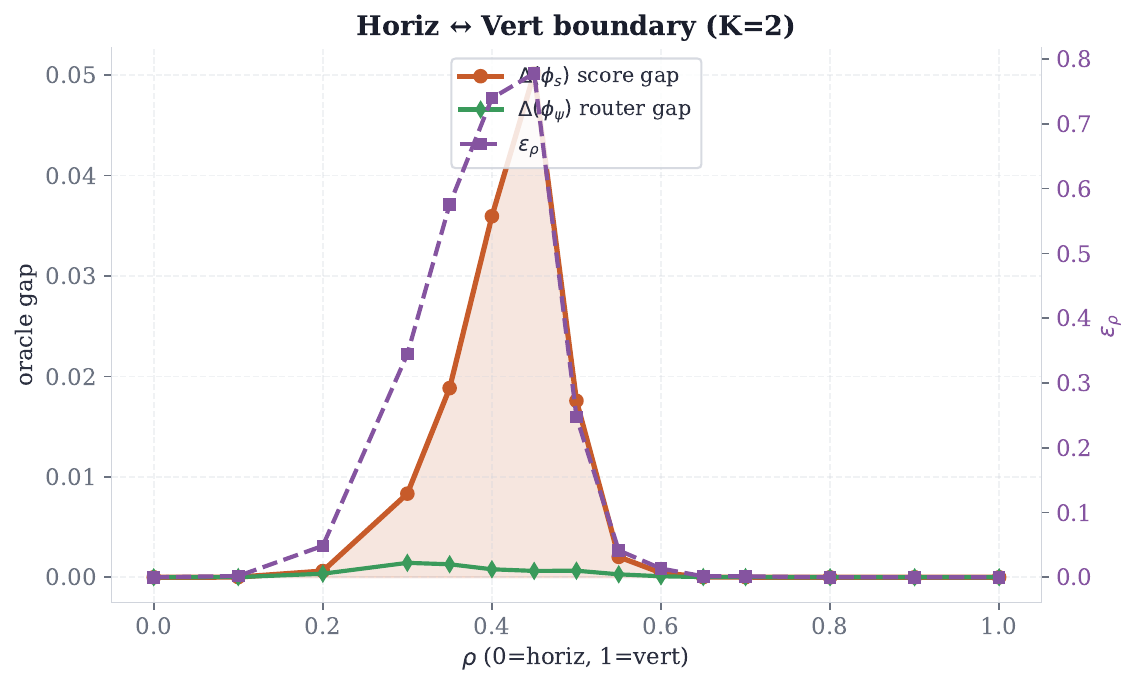}
        \caption{$K{=}2$}
        \label{fig:boundary_k2}
    \end{subfigure}
    \hfill
    \begin{subfigure}[t]{0.48\linewidth}
        \centering
        \includegraphics[width=\linewidth]{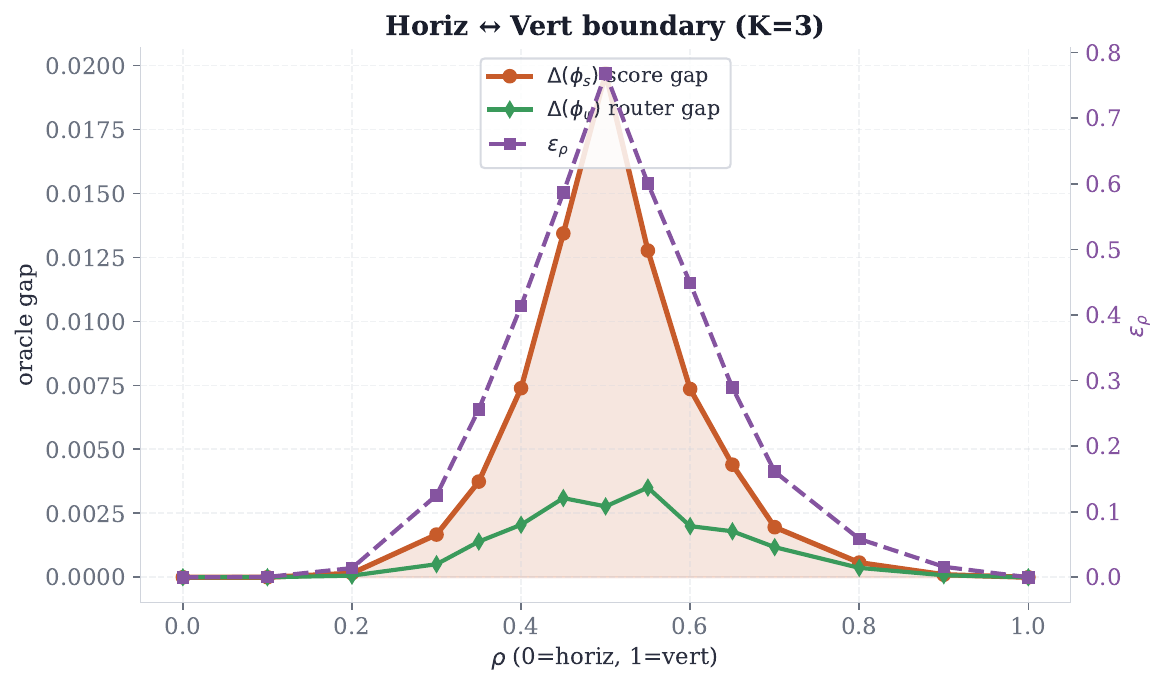}
        \caption{$K{=}3$}
        \label{fig:boundary_k3}
    \end{subfigure}
    \vspace{-0.45em}
    \begin{subfigure}[t]{0.48\linewidth}
        \centering
        \includegraphics[width=\linewidth]{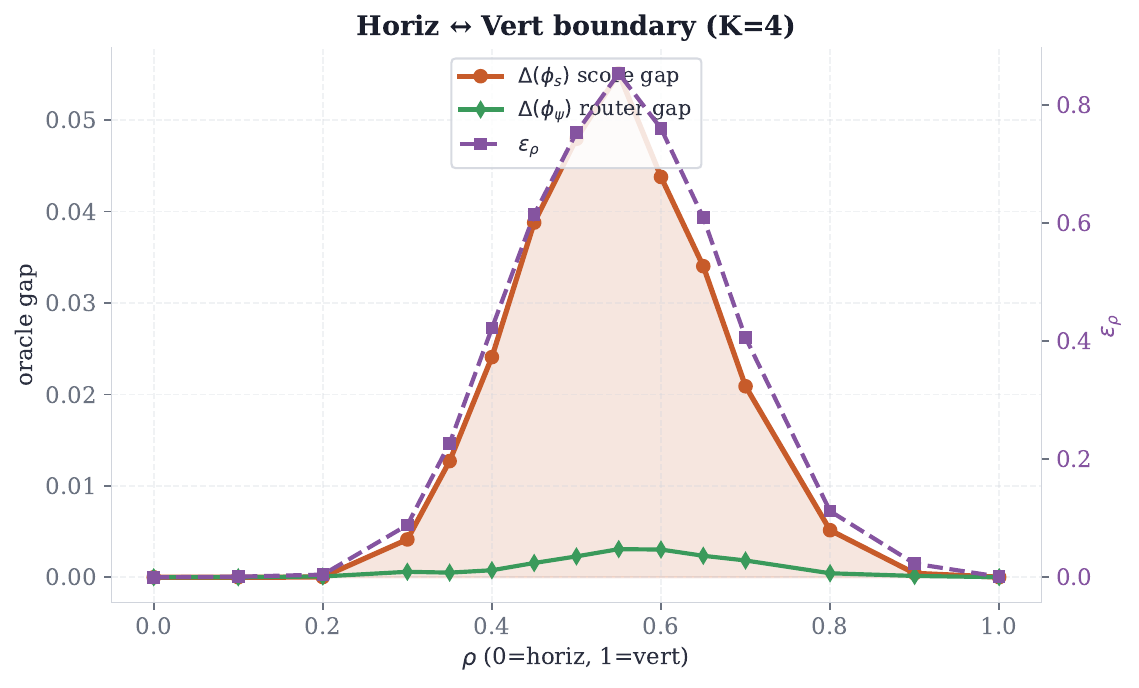}
        \caption{$K{=}4$}
        \label{fig:boundary_k4}
    \end{subfigure}
    \vspace{-0.5em}
    \caption{%
        \textbf{Gap concentration at the OOD boundary.}
        $\Delta(\phi_s)$ follows the predicted inverted-U,
        co-localized with $\varepsilon_\rho$. The router gap
        $\Delta(\phi_\psi)$ stays near zero for all three query-set
        sizes.%
    }
    \label{fig:boundary}
    \vspace{-0.8em}
\end{figure}

\begin{proposition}[Gap factorization]
\label{prop:concentration}
Let $E_\rho=\{\phi_s(x)\neq k^*(x)\}$ under
$x\sim\mathcal{D}^{\mathrm{te}}_\rho$,
$\varepsilon_\rho=P(E_\rho)$,
$\alpha_\rho=\mathbb{E}[U^*(x)\mid E_\rho]$, and
$\beta_\rho=\mathbb{E}[u_{\phi_s}(x)\mid E_\rho]$. Then
\begin{equation}
    \Delta(\phi_s;\rho)
    =
    \varepsilon_\rho(\alpha_\rho-\beta_\rho).
\end{equation}
Thus the gap concentrates wherever routing errors are frequent and the
conditional quality difference remains positive.
\end{proposition}

The equality is Proposition~\ref{prop:decomp} conditioned on
$\mathcal{D}^{\mathrm{te}}_\rho$. Figure~\ref{fig:boundary} shows the
same predicted inverted-U for each query-set size: near the training
endpoints the score head is calibrated, while at intermediate $\rho$ it
must choose among multiple specialized queries without direct
counterfactual quality supervision.
The width and height of the curves also explain why the real results vary
by dataset. The $K{=}2$ model has a sharp failure band because the
horizontal--vertical ambiguity is almost one-dimensional; the $K{=}4$
model has a broader band because extra specialized queries participate in
the candidate set. The empirical analogue is the difference between
Cityscapes and ADE20k/COCO: narrow, saturated taxonomies leave less room
for routing errors, while broad scene taxonomies expose more competing
query states. In all cases the router gap stays small near the peak,
which means the winning query is still identifiable from the frozen
output even when the default score rule is least reliable.

\FloatBarrier
\subsection{Query-Feature Routing}
\label{app:query_feature_routing}

\begin{figure}[H]
    \centering
    \begin{subfigure}[t]{0.48\linewidth}
        \centering
        \includegraphics[width=\linewidth]{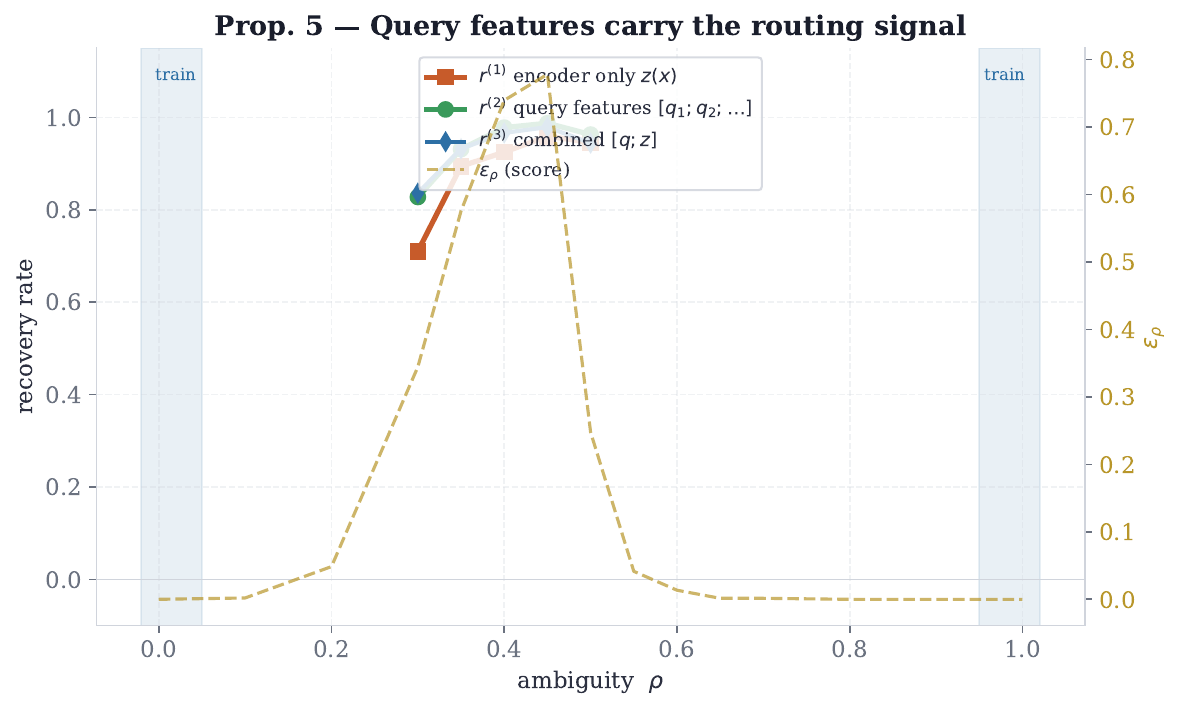}
        \caption{$K{=}2$}
        \label{fig:3router_k2}
    \end{subfigure}
    \hfill
    \begin{subfigure}[t]{0.48\linewidth}
        \centering
        \includegraphics[width=\linewidth]{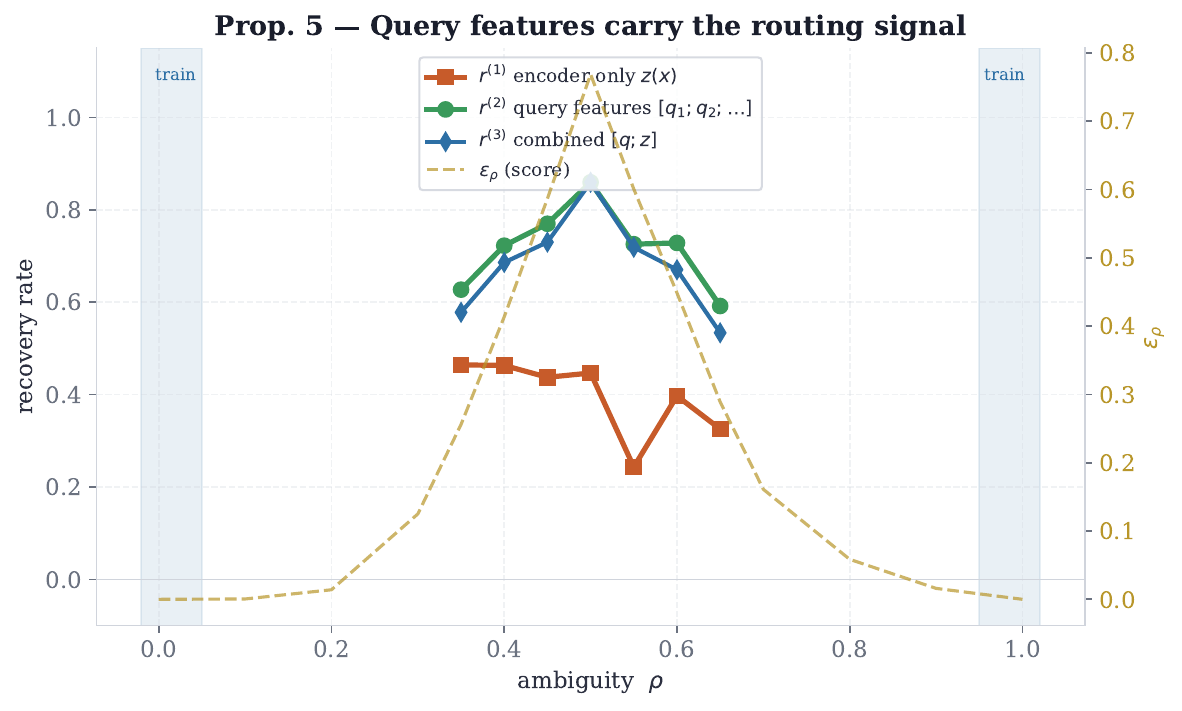}
        \caption{$K{=}3$}
        \label{fig:3router_k3}
    \end{subfigure}
    \vspace{-0.45em}
    \begin{subfigure}[t]{0.48\linewidth}
        \centering
        \includegraphics[width=\linewidth]{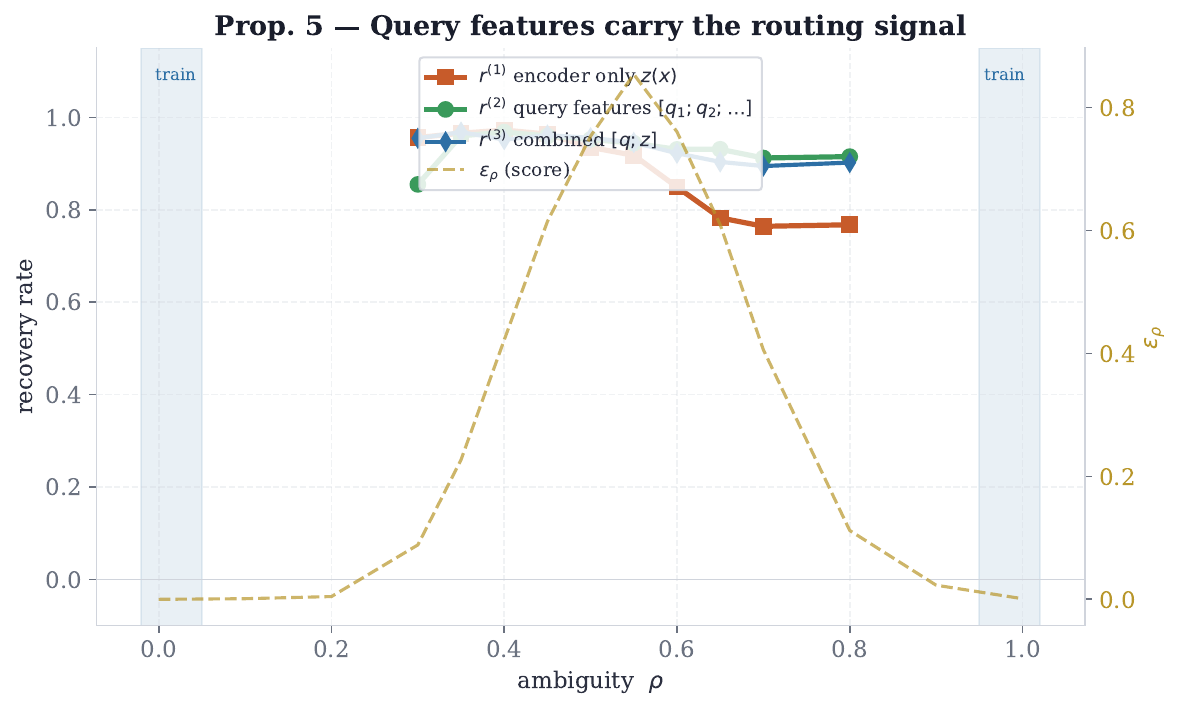}
        \caption{$K{=}4$}
        \label{fig:3router_k4}
    \end{subfigure}
    \vspace{-0.5em}
    \caption{%
        \textbf{Recovery by feature source.}
        Encoder-only routing degrades near the ambiguous boundary,
        while routers that read query features remain close to the
        oracle across $K{=}2,3,4$. The recoverable signal is carried by
        the query stream.%
    }
    \label{fig:3router}
    \vspace{-0.8em}
\end{figure}

\begin{proposition}[Query-index routing is not identifiable from the encoder alone]
\label{prop:routing}
For any encoder-only router $r^{(z)}(z(x))$ and any query-based model
$f_\theta$, there exists an equivalent model obtained by permuting the
query labels that has the same encoder representation $z(x)$ and the
same unordered candidate set, but whose oracle query index is permuted.
Therefore no encoder-only rule can be uniformly correct over equivalent
query parameterizations. Query features are permutation-equivariant and
retain the query-index information needed for post-hoc routing.
\end{proposition}

Figure~\ref{fig:3router} gives the empirical counterpart for all
controlled query-set sizes: the encoder-only router degrades at peak
ambiguity, while routers that read query features maintain high
recovery. The point is not that the toy model proves every empirical
gain; it isolates why frozen query features are the right place to look
for routing signal.
This ablation should be read carefully. The encoder representation can
contain enough information to segment the object, and in some easy
regions it can even predict the right query index by correlation. What it
lacks is a stable name for the specialized candidate after the query
slots have been learned. Query features provide that missing coordinate:
if two query slots are permuted, their features are permuted with them,
so a post-hoc rule over the query stream can remain aligned with the
candidate masks. This is why \ours{} uses the frozen output record
itself--scores, logits, masks, overlaps, and query-level descriptors--as
the calibration signal. The mechanism study therefore narrows the claim:
\ours{} is not correcting every segmentation error, but it is well matched
to the subset of errors where a better query was already computed and the
default selection chose the wrong one.

\FloatBarrier
\section{Deferred Proofs}
\label{app:proofs}

\paragraph{Proof of Proposition~\ref{prop:score}.}
By Proposition~\ref{prop:decomp},
$\Delta(\phi_s)=\mathbb{E}[\mathbf{1}[\phi_s\neq k^*](U^*-u_{\phi_s})]$.
Both factors are positive with positive probability by assumption, so
$\Delta(\phi_s)
=P(\phi_s\neq k^*)\cdot\mathbb{E}[U^*-u_{\phi_s}\mid\phi_s\neq k^*]>0$.
The score head is trained to minimize a binary detection loss that
supervises object presence rather than the counterfactual ranking of all
candidate masks by IoU. Thus agreement between $\phi_s$ and $k^*$ is not
enforced by the objective, and routing errors can persist in general.
\hfill$\square$

\paragraph{Proof of Proposition~\ref{prop:routing}.}
The encoder $z(x)$ is computed before any query-index-specific decision
is made and is invariant to relabeling of the $K$ query indices. Let
$\sigma\in\mathfrak{S}_K$ be any permutation and $f^\sigma_\theta$ the
model with query labels permuted by $\sigma$. Then $z^\sigma(x)=z(x)$
for all $x$, so an encoder-only router makes the same decision under
both parameterizations. The oracle index transforms as
$k^{\sigma *}(x)=\sigma(k^*(x))$. Unless the router is simultaneously
correct for every permutation, which is impossible when the permutation
changes the target index, no encoder-only rule can be uniformly correct
over equivalent relabelings.

Query features transform differently. If the query labels are permuted,
the sequence of query features is permuted in the same way, so the
features retain the index information needed to select the corresponding
candidate. This is the sense in which post-hoc routing over query
features is identifiable while encoder-only query-index routing is not.
\hfill$\square$

\section{Limitations, Future Work, and Notation}
\label{app:limits}

\paragraph{Limitations.}
\ours{} is supervised by oracle labels and therefore requires a labelled
calibration split. It is also a single-action selector: in the audited
SAM~3 setting it may abstain or add one query, but it does not learn a
full set-valued decoder. This restriction is deliberate for the
main paper because it isolates the selection bottleneck, but richer
decoders may recover more of the remaining oracle gap.

\paragraph{Future work.}
The single-action restriction is intentionally narrow. Natural next
steps are richer set-valued routers that can select, add, or abstain over
multiple candidates while controlling false positives; calibration
without dense labels; and selectors that expose uncertainty when the
cached candidate pool does not contain a good correction. These
directions move beyond diagnosing hidden headroom toward deployable
decoders that decide when the frozen query pool should be trusted.

\paragraph{Notation.}
$Q$ or $K$ denotes the number of query slots, $\hat{y}_0$ the default
prediction, $q^\star$ or $k^*$ the oracle query index, $U^*$ the oracle
IoU, $\Delta$ an oracle gap, and recovery the fraction of top-1 oracle
headroom captured by the learned selector.

\end{document}